\documentclass[twoside]{article}

\usepackage[accepted]{aistats2025forarxiv}
\usepackage[round]{natbib}

\usepackage{hyperref}
\usepackage{algorithm}
\usepackage{algorithmic}
\usepackage{amsmath, amsfonts}
\usepackage{booktabs}
\usepackage{multirow}
\usepackage{graphicx}
\usepackage{wrapfig}
\usepackage{thmtools}
\usepackage{thm-restate}
\usepackage{tikz}
\usetikzlibrary{arrows.meta, positioning, calc}
\def\eqref#1{equation~\ref{#1}}
\def\1{\bm{1}}

\DeclareMathAlphabet{\mathsfit}{\encodingdefault}{\sfdefault}{m}{sl}
\SetMathAlphabet{\mathsfit}{bold}{\encodingdefault}{\sfdefault}{bx}{n}

\def\gA{{\mathcal{A}}}

\def\gC{{\mathcal{C}}}
\def\gD{{\mathcal{D}}}

\def\gF{{\mathcal{F}}}
\def\gG{{\mathcal{G}}}

\def\gN{{\mathcal{N}}}
\def\gO{{\mathcal{O}}}

\def\gS{{\mathcal{S}}}

\def\gX{{\mathcal{X}}}

\newcommand{\E}{\mathbb{E}}

\newcommand{\R}{\mathbb{R}}

\DeclareMathOperator*{\argmax}{arg\,max}

\newcommand{\set}[1]{\{#1\}}

\newcommand{\brackets}[1]{\left[ #1 \right]}
\newcommand{\parens}[1]{\left( #1 \right)}

\newcommand{\abs}[1]{\left| #1 \right|}
\newcommand{\norm}[1]{\left\| #1 \right\|}
\newcommand{\infnorm}[1]{\norm{#1}_\infty}

\DeclareMathOperator{\indicator}{\mathbf{I}}

\newtheorem{theorem}{Theorem}
\newtheorem{lemma}[theorem]{Lemma}

\newenvironment{proof}{\paragraph{Proof:}}{\hfill$\square$}

\usepackage[skip=0pt]{caption}
\newcommand{\DP}{\text{DP}}
\newcommand{\pib}{\pi_\text{b}}
\newcommand{\pidp}{\pi_\text{DP}}
\newcommand{\ndp}{\tilde{n}}
\newcommand{\Mdp}{\tilde{M}}
\newcommand{\rdp}{\tilde{r}}
\newcommand{\Rdp}{\tilde{R}}
\newcommand{\gAdp}{\gA^{\DP}}
\newcommand{\gSdp}{\gS^{\DP}}
\newcommand{\Pdp}{\tilde{P}}
\newcommand{\gammadp}{\tilde{\gamma}}
\newcommand{\Vmax}{V_{\max}}

\newcommand{\Nmin}{N_{\wedge}}

\begin{document}



\twocolumn[

\aistatstitle{Decision-Point Guided Safe Policy Improvement}

\aistatsauthor{ Abhishek Sharma \And Leo Benac \And  Sonali Parbhoo \And Finale Doshi-Velez }

\aistatsaddress{ Harvard University\\Cambridge, USA \And Harvard University\\Cambridge, USA \And Imperial College\\London, UK \And Harvard University\\Cambridge, USA } ]

\begin{abstract}
  Within batch reinforcement learning, safe policy improvement seeks to ensure that the learnt policy performs at least as well as the behavior policy that generated the dataset.  The core challenge is seeking improvements while balancing risk when many state-action pairs may be infrequently visited.  In this work, we introduce Decision Points RL (DPRL), an algorithm that restricts the set of state-action pairs (or regions for continuous states) considered for improvement. DPRL ensures high-confidence improvement in densely visited states (i.e. decision points) while still 
  % FDV: How do we utilize data from sparse states?  For the trajectory-based value estimates?
  utilizing data from sparsely visited states.  By appropriately limiting where and how we may deviate from the behavior policy, we achieve tighter bounds than prior work; specifically, our data-dependent bounds do not scale with the size of the state and action spaces.  In addition to the analysis, we demonstrate that DPRL is both safe and performant on synthetic and real datasets.
\end{abstract}

\section{Introduction}
\label{sec:introduction}
Batch Reinforcement Learning (Batch RL) \citep{lange2012batch} involves developing an effective policy using a limited number of trajectories generated by another \emph{behavior policy}.  Used in settings ranging from education, robotics and medicine \citep{fuBatchReinforcementLearning2020}, Batch RL is valuable when interaction with the environment may be risky \citep{garcia2015comprehensive} or expensive \citep{kalashnikov2018scalable}. 
% the gap in setting
However, the benefits of batch RL rely on a sufficiently exploratory behavior policy \citep{sutton2018reinforcement,kumarWhenShouldWe2022}.
% motivate our setting (limited exploration/systematic errors) + diffs to the behavior
In scenarios with limited exploration and experts making systematic errors, learning an optimal policy may not be feasible without risking the adoption of an unsafe policy that performs poorly compared to the behavior. Additionally, since real-world Batch RL deployments are often incremental \citep{fuBatchReinforcementLearning2020}, it is sufficient—and perhaps preferable—to implement changes relative to the existing behavior policy. Our focus is on providing safe, high-confidence modifications in settings with limited exploration in the data.
% Motivation and gap
% We only talk about the main methods here (CQL, SPIBB, PQI)
% Other methods kostrikov2021offlinea ; kostrikov2021offlineb ; wu2018laplacian ; fujimoto2019off ; jaques2019way are discussed in the related work
% \emph{Safe} RL methods (e.g. conservative Q-learning \citep{kumarStabilizingOffPolicyQLearning2019a}, Safe Policy Improvement with Baseline Bootstrapping (SPIBB) \citep{larocheSafePolicyImprovement2019}, ) are a 

The main challenge for safe policy improvement (SPI) is appropriately restricting the learned policy to be close to the behavior data, while still identifying the points of improvement over the behavior. Density-based safe RL, such as conservative Q-learning (CQL) \citep{kumarStabilizingOffPolicyQLearning2019a} and behavior cloning (BC), constrain the learned policy to be close to behavior policy. However, this approach is too conservative if the behavior policy is both stochastic and suboptimal as it will not choose the better of the explored actions. Pessimism-based planning \citep{liuProvablyGoodBatch2020c,yuMOPOModelbasedOffline2020,kidambiMOReLModelBasedOffline2021} incorporates pessimism in the value estimates of state-action pairs (proportional to the uncertainty in the estimates), but can be overly conservative because it will penalize actions that lead to states that are not frequently visited (even if the action itself is observed often enough). 

To address the limitations of pessimistic planning and density-based constraints, support-constrained policies \cite{wu2022supported} restrict the learned policy to the support of the behavior policy. However, these policies can become unsafe with even a few noisy actions or rewards, which are common in practice. Among support-based methods, count-based techniques like Safe Policy Improvement with Baseline Bootstrapping (SPIBB) \cite{larocheSafePolicyImprovement2019} impose a count-based constraint on $(s,a)$ pairs to ensure policy improvements occur only over sufficiently observed state-action pairs. However, SPIBB requires access to the behavior policy function, which is often impractical in real-world applications. We demonstrate that its performance significantly deteriorates when the behavior policy must be estimated, and its improvement guarantees are not tight in practice.

Our key insight is that we need to identify only those behavior changes that yield the most improvement. We achieve this by constraining changes to $(s,a)$ pairs observed frequently in the dataset, specifically those with a count $\geq \Nmin$ (termed `Decision Points'). For states where we lack high-confidence improvements, we defer to the current behavior policy. This approach enhances performance over SPIBB in two ways: it allows for better returns without needing access to the true behavior policy and enables an explicit `DEFER' flag when we are unsure about achieving safe improvements. The hyperparameter $\Nmin$ allows us to balance confidence in performance improvement, safety, and the number of changes. Additionally, this method results in a learned policy with a few high-improvement changes that can be easily reviewed and implemented in practice (e.g., by a clinician).
% FDV: Original paragraph 
% In this paper, our goal is to identify safe changes to improve current practice.  We consider the following desiderata as proposed in both RL and cognitive science literature e.g. \cite{balazadehLearningSwitchAgents2023,steyvers2023three}: the improved policy should be (i) safe, (ii) highly-performant, 
% FDV: I'm trying to figure out whether and how (iii) should be highlighted 
% (iii) recommend few changes to current practice by explicitly defer to the behavior policy in situations where no change is recommended. We follow a similar approach to SPIBB \citep{larocheSafePolicyImprovement2019} in that we constrain the policy to only suggest only those state-action pairs that are observed enough in the dataset, controlled by a hyperparameter $\Nmin$. When the improvement is not possible with high confidence, we provide an explicit "defer" flag. This allows us to trade off confidence of performance improvement with safety and the number of changes using a single hyperparameter.
% By keeping the changes to be minimal (yet high-improvement) and deferring explicitly, we reduce the cognitive load of the practitioner (e.g., a physician) who wishes to learn from the recommended changes over the current practice.
% (Thighten version)

Our work makes the following contributions: (i) We introduce Decision Points RL (DPRL), a safe batch RL algorithm that restricts improvements to specific state-action pairs (or regions in continuous states), deviating from the behavior policy only when confident and deferring otherwise. (ii) Unlike previous methods, DPRL operates without needing knowledge of the behavior policy to ensure safety. (iii) DPRL offers significantly tighter theoretical guarantees than existing algorithms for both discrete and continuous state problems, with bounds that depend on the safety threshold parameter $\Nmin$ rather than the state-action space size. (iv) Empirically, we demonstrate that DPRL better balances safety and improvement compared to alternatives in both synthetic and real-world medical datasets while making minimal changes.  

% Specifically, our work makes the following contributions [TODO: TIGHTEN THIS TEXT]: 
% \begin{itemize}
%     \item We introduce Decision Points RL (DPRL) a safe batch RL algorithm which restricts the set of state-action pairs (or regions for continuous states) considered for improvement, deviates from the behavior policy only if improvement is possible with high confidence, and otherwise explicitly defers to the behavior policy.
%     \item Unlike prior work, we do not rely on knowing the behavior policy to obtain a safe policy.
%     \item The construction of DPRL results in theoretical guarantees that are significantly tighter than existing algorithms across both discrete and continuous state problems. Notably, our bounds do not depend on the size of state and action spaces when the dataset covers only a subset of the full spaces, or any other bounds on the data distribution, and are expressed directly in terms of the safety threshold parameter $\Nmin$.
%     \item Finally, we demonstrate empirically on a synthetic domain and a real-world medical dataset that our algorithm better trades safety and improvement than alternatives, while also being parsimonious in its changes.
% \end{itemize}

\section{Related Work}
\label{sec:related_work}
% FDV: Clean up -- there's citations below that are termed SPI in the intro.  We need to be clear about what makes something Policy Reg vs. SPI, or maybe combine?  
\textbf{Policy regularization in Offline RL.}
Many RL applications require agents to learn from a fixed batch of pre-collected data, limiting further data collection. Various methods constrain policies to this data. Density-constraining methods, such as \citep{fujimotoOffPolicyDeepReinforcement2019a, kumarStabilizingOffPolicyQLearning2019a, kumarConservativeQLearningOffline2020, yuCOMBOConservativeOffline2022, thomasHighConfidencePolicy2015a}, keep the action space close to the behavior policy. For instance, \citep{kumarStabilizingOffPolicyQLearning2019a} constrains action selection based on bootstrapping errors for actions outside the training data, while \citep{kumarConservativeQLearningOffline2020} introduces a conservative Q-learning penalty (CQL) to address distribution shifts between the dataset and the learned policy. 
However, we demonstrate that density regularization techniques can be suboptimal when the behavior policy is stochastic and suboptimal in certain states. In contrast, our approach does not impose restrictions on the policy distribution, only limiting support to frequently observed actions. Support-constraining methods, such as \citep{singhOfflineRLRealistic2022a}, restrict actions to those within the support of the behavior data but lack guarantees. Unlike our method, policy regularization techniques in offline RL do not prioritize safe policy improvement.

\textbf{Safe Policy Improvement.} 
Several studies have addressed safe policy improvement in batch RL settings (e.g., \citep{thomasHighConfidenceOffPolicyEvaluation2015, ghavamzadehSafePolicyImprovement2016, larocheSafePolicyImprovement2019}). \citep{ghavamzadehSafePolicyImprovement2016} utilizes pessimism to regularize infrequently occurring state-action pairs. Meanwhile, \citep{larocheSafePolicyImprovement2019,nadjahiSafePolicyImprovement2019} propose an algorithm that bootstraps a trained policy with a baseline when uncertainty is high, offering SPI guarantees only when insufficiently observed pairs adhere to the behavior policy. However, these results only hold if we have access to the behavior policy a priori. Other works, such as \cite{schollSafePolicyImprovement2022} and \cite{wienhoftMoreLessSafe2023}, provide additional guarantees for bootstrapping methods. Some researchers have used pessimism to regularize the reward or action-value function for rarely observed pairs. \citep{liuProvablyGoodBatch2020c} presents guarantees for batch RL via pessimistic formulations of policy and Q-iteration algorithms, while \citep{kidambiMOReLModelBasedOffline2021} and \citep{yuMOPOModelbasedOffline2020} focus on learning pessimistic MDPs for near-optimal policies. \citep{kimModelbasedOfflineReinforcement2023b} constructs a penalized reward function based on state-action counts but fails to exclude rarely observed pairs. In contrast, our approach does not regularize the value function or reward model but directly constrains the policy, avoiding excessive conservativeness. Notably, we do not require explicit access to the behavior policy and provide tighter improvement guarantees.

\textbf{Non-parametric RL and Trajectory Stitching.} 
Instead of directly constraining actions and states for safety and conservatism, there is extensive research on improving models through nonparametric methods or heuristics in areas with uneven data coverage. Among these, \cite{shresthaDeepAveragersOfflineReinforcement2020} use k-nearest neighbors to estimate reward and transition models, which works well when neighbors are nearby but fails if neighbors are distant, while others such as \cite{gottesmanCombiningParametricNonparametric2019} combine parametric and nonparametric methods for better policy evaluation. Unlike these, our work makes no assumptions of the form of value function and offers a general analysis for safe policy improvement. Other methods, like those by \cite{charBATSBestAction2022,hepburnModelbasedTrajectoryStitching2022}, use distance-based metrics for value function approximation but do not focus on safety policy improvement and assume accurate model estimation, which is challenging with uneven data coverage. \citep{zhangInterpretableRLFramework2022} introduce decision regions for safe policy learning using a heuristic approach without theoretical guarantees, unlike our method. Notably, none of these heuristic or nonparametric approaches offer any theoretical guarantees as we do here.
More closely related, \citep{ledererUniformErrorBounds2019} use Gaussian Processes (GPs) to estimate transition models. Though their approach does provide theoretical bounds, our guarantees are tighter; we also focus explicitly on the task of safe policy improvement.
 
\textbf{Learning to Defer to Human Expertise.}
Policy regularization, SPI, and trajectory stitching methods do not effectively allow for selective use of the behavior policy during decision-making or deferring to it when the learned policy may not be significantly better. Some studies focus on learning to defer to human expertise in both static and sequential contexts. For instance, \citep{madras2018predict} and \citep{mozannar2020consistent} investigate this in static classification tasks, while \citep{liKnowsWhatIt2011} explore it in online settings that require a polynomial number of deferrals. In offline settings, however, more frequent deferrals may be more reasonable, as proposed here.
% \citep{straitouriReinforcementLearningAlgorithmic2021} study deferring to the behavior policy when uncertain but lack theoretical guarantees. \citep{joshiLearningtodeferSequentialMedical2022} also explore deferring in nonstationary, sequential settings with offline data, learning a policy that defers when its value is lower or more uncertain than the behavior policy, but without providing theoretical guarantees on the learned policy's quality.
\citep{straitouriReinforcementLearningAlgorithmic2021, joshiLearningtodeferSequentialMedical2022} also study deferring to the behavior policy when uncertain but lack theoretical guarantees.

\section{Background}
\label{sec:background}

% \paragraph{Markov Decision Processes (MDP).} 
An MDP is a tuple $(\mathcal{S}, \mathcal{A}, P, R, \gamma)$ of discount $\gamma \in [0,1)$, states $s \in \mathcal{S}$, actions $a \in \mathcal{A}$, transition probabilities $P(s'|s,a)$, and rewards $R: \mathcal{S} \times \mathcal{A} \to \mathbb{R}$.  In this work, we assume that $R$ is bounded in $[0, R_{\max}]$, $\gA$ is discrete, the starting state is fixed, and that $R$ and $P$ are unknown. 

The behavior policy $\pi_b(a|s)$ is the policy that generated the observed trajectories, and $\pi(a|s)$ is the policy we are trying to learn. Given a policy $\pi$, we call $\eta^\pi_h(s) = \Pr[S_h = s | \pi]$ the marginal distribution of $S_h$ under $\pi$. Then,
$\eta^\pi(s,a) = \eta^\pi_h(s) \pi(a|s) = (1-\gamma) \sum_{h=0}^\infty \gamma^h \eta^\pi_h(s,a)$ is called the marginal distribution of $(s,a)$.
We are given a dataset $\gD = \set{{S^n_0, A^n_0, R^n_0, \cdots, S^n_{T_n}, A^n_{T_n}, R^n_t}}_{n=1}^N$ consisting of $N$ trajectories, with actions taken by $\pib$. State-action pairs in $\gD$ can also be assumed to drawn \emph{i.i.d.} from a \emph{behavior distribution} $\mu(s,a) = \eta^\pi(s,a)$. We overload the notation and also denote the marginal distribution over states by $\mu(s) = \sum_{a \in \gS)} \mu(s,a)$.

The \emph{value} $V_{\pi}(s)$ of policy $\pi$ at state $s$ is the expected discounted sum of rewards starting from $s$ following policy $\pi$: $V_\pi(s) = \mathbb{E}_{\pi}[\sum_{t=1}^{T} \gamma^{t-1} R_{t}|s_1 = s]$.
The action-value function (or Q-function) $Q_{\pi}(s,a) $ is the value of performing action $a$ in state $s$ and then performing policy $\pi$ after: $Q_\pi(s,a) = \mathbb{E}_\pi[\sum_{t=1}^{T} \gamma^{t-1} R_{t}|s_1 = s, a_1 = a]$. Let $Q_{\max}$ or $V_{\max}$ be upper bounds on $Q(s,a)$ and $V(s)$. We denote $\rho(\pi)$ as the value of the start state.

\section{Challenges with Prior Algorithms}
\label{sec:mdp_example}

\begin{figure*}[ht]
    \centering
    \begin{minipage}[b]{0.45\linewidth}
        \centering
        \resizebox{.9\textwidth}{!}{\begin{tikzpicture}[
      % Define styles for different node types
      node distance=0.7cm and 0.4cm, % Vertical and horizontal spacing between nodes
      state/.style={draw, circle, minimum size=8mm, inner sep=0pt}, % Circular state nodes
      action/.style={fill, circle, minimum size=2pt, inner sep=2pt}, % Small filled action nodes
      square/.style={draw, rectangle, minimum size=5mm, inner sep=0pt, fill=gray!50}, % Square node with gray fill
      dots/.style={fill, circle, minimum size=1pt, inner sep=1pt}, % Small dots for intermediary connections
      every edge/.style={draw, -{Stealth[round]}, thick}, % Arrow style
      >=Stealth
    ]

    % ------------------------------
    % Initial Node
    % ------------------------------
    \node[state] (s0) {$s_0$}; % Starting state of the MDP

    % ------------------------------
    % Action Nodes
    % ------------------------------
    \node[action] (a0) [above=1cm of s0, xshift=1.7cm, yshift=.4cm] {}; % Action leading to upper branches
    \node[action] (a1) [left=.2cm of s0] {}; % Action leading to middle branch
    \node[action] (a2) [below=1cm of s0,  xshift=1.7cm, yshift=-.2cm] {}; % Action leading to lower branches

    % ------------------------------
    % Connections from Initial Node to Actions
    % ------------------------------
    \draw[] (s0) -- (a0) node[midway, above left] {$\pi_b(a_0)=\epsilon$}; % Edge from s0 to a0 with label a0
    \draw[] (s0) -- (a1) node[midway, above left] {$a_1$}; % Edge from s0 to a1 with label a1
    \draw[] (s0) -- (a2) node[midway, above left] {$\pi_b(a_2)=\epsilon$}; % Edge from s0 to a2 with label a2

    % ------------------------------
    % Final Node
    % ------------------------------
    \node[square] (z) [right=7cm of a1] {}; % Final square node where all chains converge
    \node[square] (z1) [left=4.5cm of a1] {};

    % ------------------------------
    % Upper Branches (Action a0)
    % ------------------------------
    % Chain 1: s1D -> s2D -> sD1 -> z
    \node[state] (s1D) [right=of a0, yshift=0.8cm] {$s_{D1}$}; % First state in Chain 1 (indexed by D)
    \draw[->] (a0) -- (s1D); % Edge from a0 to s1D

    \node[state] (s2D) [right=of s1D] {$s_{21}$}; % Second state in Chain 1 (indexed by D)
    \draw[->] (s1D) -- (s2D); % Edge from s1D to s2D

    \node[state, fill=green!30] (sD1) [right=of s2D] {$s_{D1}$}; % Third state in Chain 1 (indexed by K)
    \draw[->] (s2D) -- (sD1); % Edge from s2D to sD1

    \draw[->] (sD1) -- (z); % Edge from sD1 to final node z

    % Chain 2: s1K -> s2K -> sDK -> z
    \node[state] (s1K) [right=of a0, yshift=-1cm] {$s_{1K}$}; % First state in Chain 2 (indexed by K)
    \draw[->] (a0) -- (s1K); % Edge from a0 to s1K

    \node[state] (s2K) [right=of s1K] {$s_{2K}$}; % Second state in Chain 2 (indexed by K)
    \draw[->] (s1K) -- (s2K); % Edge from s1K to s2K

    \node[state, fill=green!30] (sDK) [right=of s2K] {$s_{DK}$}; % Third state in Chain 2 (indexed by K)
    \draw[->] (s2K) -- (sDK); % Edge from s2K to sDK

    \draw[->] (sDK) -- (z); % Edge from sDK to final node z
    \node[above] at (sD1.north) {$r \sim \textbf{Unif}[.65, .75]$};

    % ------------------------------
    % Intermediary Dots Between Upper Chains
    % ------------------------------
    % Dots between s1D and s1K
    \node[dots] at ($(s1D)!0.4!(s1K)$) {}; % First dot
    \node[dots] at ($(s1D)!0.5!(s1K)$) {}; % Second dot
    \node[dots] at ($(s1D)!0.6!(s1K)$) {}; % Third dot

    % Dots between s2D and s2K
    \node[dots] at ($(s2D)!0.4!(s2K)$) {}; % First dot
    \node[dots] at ($(s2D)!0.5!(s2K)$) {}; % Second dot
    \node[dots] at ($(s2D)!0.6!(s2K)$) {}; % Third dot

    % Dots between sD1 and sDK
    \node[dots] at ($(sD1)!0.4!(sDK)$) {}; % First dot
    \node[dots] at ($(sD1)!0.5!(sDK)$) {}; % Second dot
    \node[dots] at ($(sD1)!0.6!(sDK)$) {}; % Third dot

    % ------------------------------
    % Middle Branch (Action a1)
    % ------------------------------
    \node[state] (f1) [left=of a1] {$f_1$}; % First state in middle chain
    \draw[->] (a1) -- (f1); % Edge from a1 to f1

    \node[state] (f2) [left=of f1] {$f_2$}; % Second state in middle chain
    \draw[->] (f1) -- (f2); % Edge from f1 to f2

    \node[state] (fD) [left=of f2] {$f_{D}$}; % Final state in middle chain (indexed by D)
    \draw[->] (f2) -- (fD); % Edge from f2 to fD

    \draw[->] (fD) -- (z1); % Edge from fD to final node z
    \node[below] at (fD.south) {$r=0.55$};

    % ------------------------------
    % Lower Branches (Action a2)
    % ------------------------------
    % Chain 1: t1D -> t2D -> tD1 -> z
    \node[state] (t1D) [right=of a2, yshift=1.2cm] 
    {$t_{11}$}; % First state in Chain 1 (indexed by D)

    \draw[->] (a2) -- (t1D); % Edge from a2 to t1D

    \node[state] (t2D) [right=of t1D] {$t_{21}$}; % Second state in Chain 1 (indexed by D)
    \draw[->] (t1D) -- (t2D); % Edge from t1D to t2D

    \node[state, fill=red!30] (tD1) [right=of t2D] {$t_{D1}$}; % Third state in Chain 1 (indexed by K)

    \draw[->] (t2D) -- (tD1); % Edge from t2D to tD1

    \draw[->] (tD1) -- (z); % Edge from tD1 to final node z

    % Chain 2: t1K -> t2K -> tDK -> z
    \node[state] (t1K) [right=of a2, yshift=-0.7cm] {$t_{1K}$}; % First state in Chain 2 (indexed by K)
    \draw[->] (a2) -- (t1K); % Edge from a2 to t1K

    \node[state] (t2K) [right=of t1K] {$t_{2K}$}; % Second state in Chain 2 (indexed by K)
    \draw[->] (t1K) -- (t2K); % Edge from t1K to t2K

    \node[state, fill=red!30] (tDK) [right=of t2K] {$t_{DK}$}; % Third state in Chain 2 (indexed by K)
    \draw[->] (t2K) -- (tDK); % Edge from t2K to tDK
    \node[below] at (tDK.south) {$r \sim \textbf{Unif}[0, 1]$};
    \draw[->] (tDK) -- (z); % Edge from tDK to final node z

    % ------------------------------
    % Intermediary Dots Between Lower Chains
    % ------------------------------
    % Dots between t1D and t1K
    \node[dots] at ($(t1D)!0.4!(t1K)$) {}; % First dot
    \node[dots] at ($(t1D)!0.5!(t1K)$) {}; % Second dot
    \node[dots] at ($(t1D)!0.6!(t1K)$) {}; % Third dot

    % Dots between t2D and t2K
    \node[dots] at ($(t2D)!0.4!(t2K)$) {}; % First dot
    \node[dots] at ($(t2D)!0.5!(t2K)$) {}; % Second dot
    \node[dots] at ($(t2D)!0.6!(t2K)$) {}; % Third dot

    % Dots between tD1 and tDK
    \node[dots] at ($(tD1)!0.4!(tDK)$) {}; % First dot
    \node[dots] at ($(tD1)!0.5!(tDK)$) {}; % Second dot
    \node[dots] at ($(tD1)!0.6!(tDK)$) {}; % Third dot

    \end{tikzpicture}
%     \caption{Toy MDP for PQI. The optimal action ($a_0$) leads to the upper branch with two parallel chains indexed by depth ($D$) and chain number ($K$), $a_1$ leads to the middle branch, and $a_2$ leads to the lower branch with two parallel chains indexed by depth ($D$) and chain number ($K$). All branches converge to the final state $z$. Intermediary dots indicate additional potential connections between parallel chains.}
%     \label{fig:toy_mdp_pqi}
% \end{figure}
 }
    \end{minipage}
    \hspace{0.05\linewidth} % Horizontal space between the two figures
    \begin{minipage}[b]{0.45\linewidth}
        \centering
        \resizebox{0.65\textwidth}{!}{\begin{tikzpicture}[
      % vertical and horizontal node distances
      node distance=0.7cm and 0.4cm,
      state/.style={draw, circle, minimum size=8mm, inner sep=0pt}, % define state nodes and apply the circular style
      action/.style={fill, circle, minimum size=2pt, inner sep=2pt}, % define action nodes (small filled circles)
      square/.style={draw, rectangle, minimum size=5mm, inner sep=0pt, fill=gray!50}, % define square node style
      dots/.style={fill, circle, minimum size=1pt, inner sep=1pt}, % smaller dots
      every edge/.style={draw, -{Stealth[round]}, thick},
    ]

    % Center node S_0
    \node[state] (s0) at (0, 0) {$s_0$};

    % Nodes b1 and b2 to the right of s0
    \node[state, fill=green!30] (b1) [right=1cm of s0, yshift=1cm] {$b_1$};
    \node[state] (b2) [below=of b1] {$b_2$};

    \node[above] at (b1.north) {$r \sim \textbf{Unif}[0.5, 0.9]$};
    \node[below] at (b2.south) {$r = 0.55$};

    % Nodes c1 and cK to the left of s0
    \node[state, fill=red!30] (c1) [left=1cm of s0, yshift=1cm] {$c_1$};
    \node[state, fill=red!30] (cK) [below=1cm of c1] {$c_K$};

    \node[above] at (c1.north) {$r \sim \textbf{Unif}[0, 1]$};
    \node[below] at (cK.south) {$r \sim \textbf{Unif}[0, 1]$};

    % Square node aligned vertically between b1 and b2, and c1 and cK
    \node[square] (sq_b) at ($(b1)!0.5!(b2)$) [right=0.7cm] {}; % square aligned vertically between b1 and b2
    \node[square] (sq_c) at ($(c1)!0.5!(cK)$) [left=0.7cm] {};  % square aligned vertically between c1 and cK

    % Arrows between s0 and b1, b2
    \draw[->] (s0) -- (b1) node[midway, above] {$\epsilon$};
    \draw[->] (s0) -- (b2) node[midway, above] {$1-2\epsilon$};

    % Arrows between s0 and c1, cK
    \draw[->] (s0) -- (c1) node[midway, above] {$\epsilon/K$};
    \draw[->] (s0) -- (cK) node[midway, above] {$\epsilon/K$};

    % Arrows from b1 and b2 to square node on the right
    \draw[->] (b1) -- (sq_b);
    \draw[->] (b2) -- (sq_b);

    % Arrows from c1 and cK to square node on the left
    \draw[->] (c1) -- (sq_c);
    \draw[->] (cK) -- (sq_c);

    % Manually add 3 smaller, closer non-action dots between c1 and cK
    \node[dots] at ($(c1)!0.35!(cK)$) {}; % 35% of the way from c1 to cK
    \node[dots] at ($(c1)!0.50!(cK)$) {}; % 50% of the way from c1 to cK
    \node[dots] at ($(c1)!0.65!(cK)$) {}; % 65% of the way from c1 to cK

\end{tikzpicture}
%     \caption{Toy MDP for CQL. The optimal action leads to state $b_1$, and the suboptimal action (frequent under behavior) leads to state $b_2$. There are also risky states $\set{c_1, \dots, c_K}$. TODO: LABEL THE ACTIONS IN THE MDP.}
%     \label{fig:toy_mdp_cql}
% \end{figure}
}
    \end{minipage}
    \vspace{-0.02\linewidth} % Horizontal space between the two figures

    \includegraphics[width=.24\textwidth]{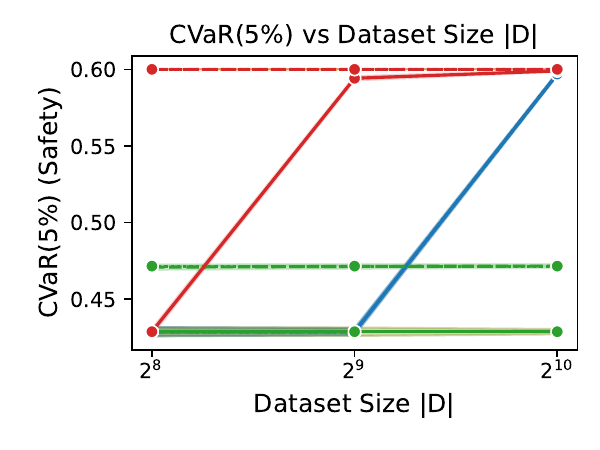}
    \includegraphics[width=.24\textwidth]{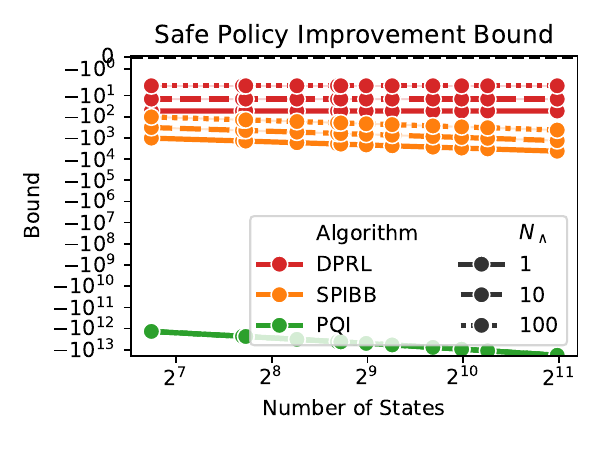}
    \includegraphics[width=.24\textwidth]{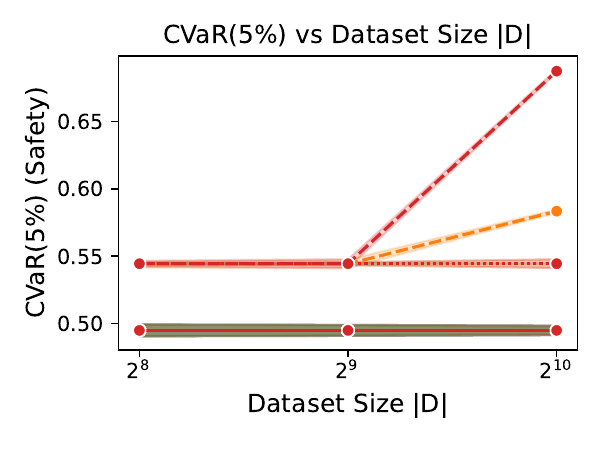}
    \resizebox{.21\textwidth}{!}{\includegraphics[width=.24\textwidth]{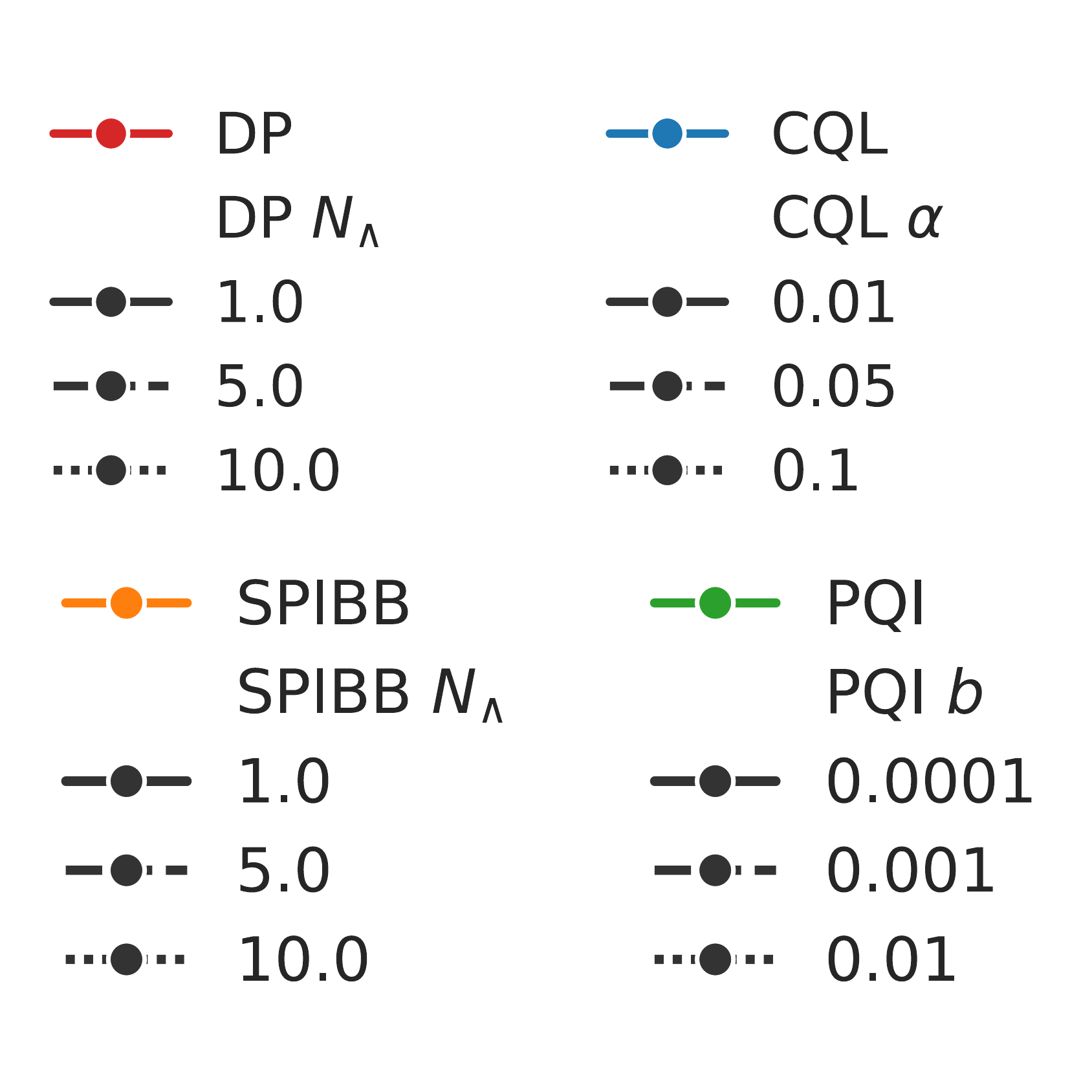}}
    \caption{Top: Challenging MDPs prior approaches. In the first MDP, the behavior goes left with high probability. In the second MDP, the behavior goes to state $b_2$ with high probabilit. Green states are part of optimal trajectories, and red states are part of risky trajectories. Bottom L to R: PQI does poorly on the first MDP, DP has the tightest safety bounds, and CQL does poorly on the second MDP.}
    \label{fig:mdp_example}
\end{figure*}
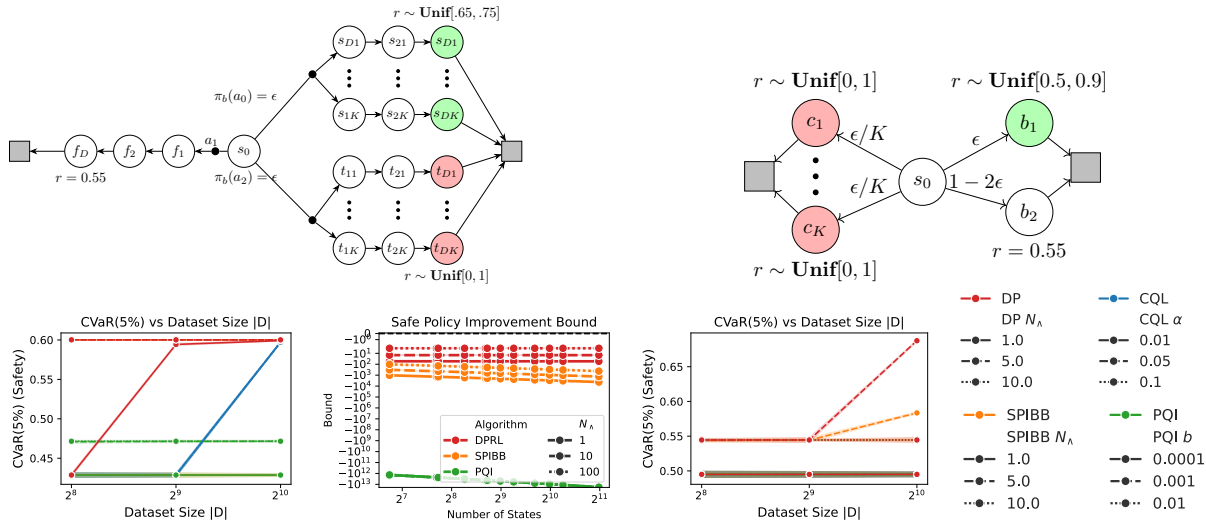

% \begin{figure*}
%     \centering
%     % \includegraphics[width=\textwidth]{figures/mdp_fig.pdf}
%     \includegraphics[width=\textwidth]{figures/MDP_PLOT_TEMP.png}
%     \caption{(a) Illustration of MDP Example, (b) Performance improvement bounds in comparison to baselines, (c) Safety-performance trade-off of DPRL in comparison to baselines. TODO: Toy MDPs at \textsc{figures/toy\_mdp\_for\_pqi} and \textsc{figures/toy\_mdp\_for\_cql}. THEY WILL LIKELY BE CALLED "FORESTS" AND "RISKY-ARMS" MDPS.}
%     \label{fig:mdp_example}
% \end{figure*}
% assumptions about the data distribution (concentrability ; bounded density (PQI))
% concentrability and its limitations : can be arbitrarily large ; algorithms often diverge in practice
% bounded density and its limitations : can be arbitrarily loose when the behavior is far from optimal in some states
% assumptions about access to behavior (SPI + follow-up work): not practical in many real-world settings
% other methods (Count-MOReL) assume uniform exploration to ensure that the error in the transition model is small ; not practical in many real-world settings

% with function approximation: 
% assumptions about the Q model class: bellman completeness 
% FDV: We haven't defined the notation!  Define as you go inline 

\textbf{Issues with Standard Assumptions.}
Many existing algorithms make assumptions about the data distribution. 
One common assumption is concentrability, i.e., $\infnorm{\nu(s,a)/\mu(s,a)} \leq C$, where $\nu(s,a)$ is any distribution reachable for some non-stationary policy. The improvement guarantee scales with $C$. However, $C$ can be arbitrarily large for real-world dataset, and algorithms relying on this assumption often diverge in practice \citep{liuProvablyGoodBatch2020c}.
To remedy this, \citep{liuProvablyGoodBatch2020c} assumed bounded density instead, i.e., $\eta^{\pi}_h(s,a) \leq U$ for any non-stationary policy $\pi$. The improvement guarantee scales with $U/b$, where $\indicator[{b \leq \hat{\mu}(s,a)}]$ is the filter chosen to exclude low-density $(s,a)$ pairs in the empirical data density $\hat{\mu}$. However, the improvement guarantee is arbitrarily loose when selecting $(s,a)$ pairs with sufficient observations but low value under the behavior policy.

\textbf{Illustrative Example.} We now show how existing SPI baselines (PQI, SPIBB, CQL) fail to choose optimal actions while avoiding risky ones. To demonstrate how a pessimism-based approach fails, consider the first MDP in Figure~\ref{fig:mdp_example}. There are three actions at the starting state $s_0$: an optimal action, $a_0$, leading to a ‘forest’ of sparsely-visited states with good outcomes; a risky action, $a_2$, leading to a ‘forest’ of states with bad outcomes; and a suboptimal, risk-free action, $a_1$, that leads to a middle road and is chosen by the behavior most of the time. A pessimism-based approach like PQI---which penalizes value estimates based on uncertainty---fails to learn that, even though the (\(s,a\)) pairs in the optimal forest are sparse in the data, we have enough observations to conclude that the good forest is the better choice. It should be noted that reducing the threshold for (\(s,a\)) density may result in taking the risky action.

For the second scenario, suppose there are 10 actions, 8 of which lead to extremely low values (and hence are risky), and it is difficult to determine that they are risky. One action, $a_1$, is optimal, while another, $a_2$, is frequently chosen by the behavior but yields suboptimal rewards. A density regularization method like CQL chooses the suboptimal action if it increases its regularization penalty (i.e. favoring the behavior policy) and chooses the risky actions if it reduces the penalty.

These two MDP scenarios illustrate how current SPI methods struggle with different types of action-risk dynamics. This intuition is supported empirically. In Figure~\ref{fig:mdp_example} (left plot), we see that PQI fails to attain a high 5\% conditional value at risk (CVaR 5\%) (a metric used to measure safe improvement) even for relatively large dataset sizes in the ``forest" MDP. Similarly, in the second MDP (right plot), CQL and SPIBB are not able to consistently learn a safe policy. In both cases, we can set an $N_{min}$ parameter in DPRL to ignore actions that cannot be reliably estimated to have a good value.
These examples also illuminate the settings in which DPRL performs particularly well: when there are systematic deviations from the behavior policy and when we are not required to act at all points in the state-action space.

\textbf{Access to behavior.}
While SPIBB appears to be similar to DPRL in its use of a count parameter like $\Nmin$, it requires access to the true behavior policy at training time---like most SPI approaches \citep{larocheSafePolicyImprovement2019,schollSafePolicyImprovement2022,wienhoftMoreLessSafe2023a}. We show in Section~\ref{sec:experiments} that its performance can be significantly worse without access to the behavior policy.
In contrast, we do not require access to the behavior policy during training time. This is an important consideration in real-world systems, where it can be difficult to elicit the behavior policy. For example, when working with doctors, we cannot expect to know the functional form of their behavior.
However, we do require access to the behavior policy for \textit{evaluation}, which can be achieved by either running a silent trial or by using off-policy evaluation (OPE). OPE in a learning-to-assist framework is a promising area of research, and beyond the scope of this work.

\section{Method}
\label{sec:method}
We now describe our decision-point RL method, which addresses several of the challenges with existing algorithms.  We provide bounds for this method in Sec.~\ref{sec:analysis}.

\begin{figure*}
    \includegraphics[width=.34\textwidth]{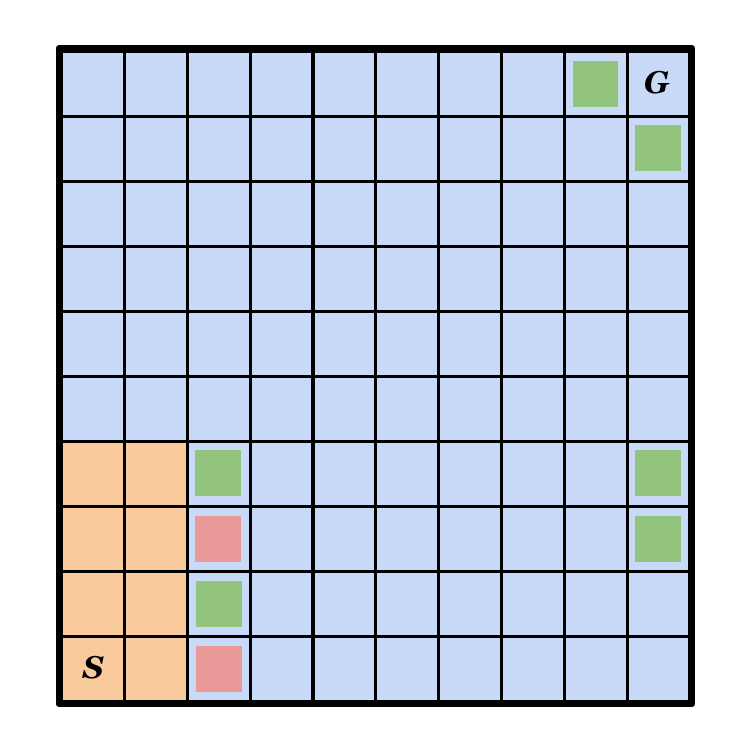}
    \includegraphics[width=.3\textwidth]{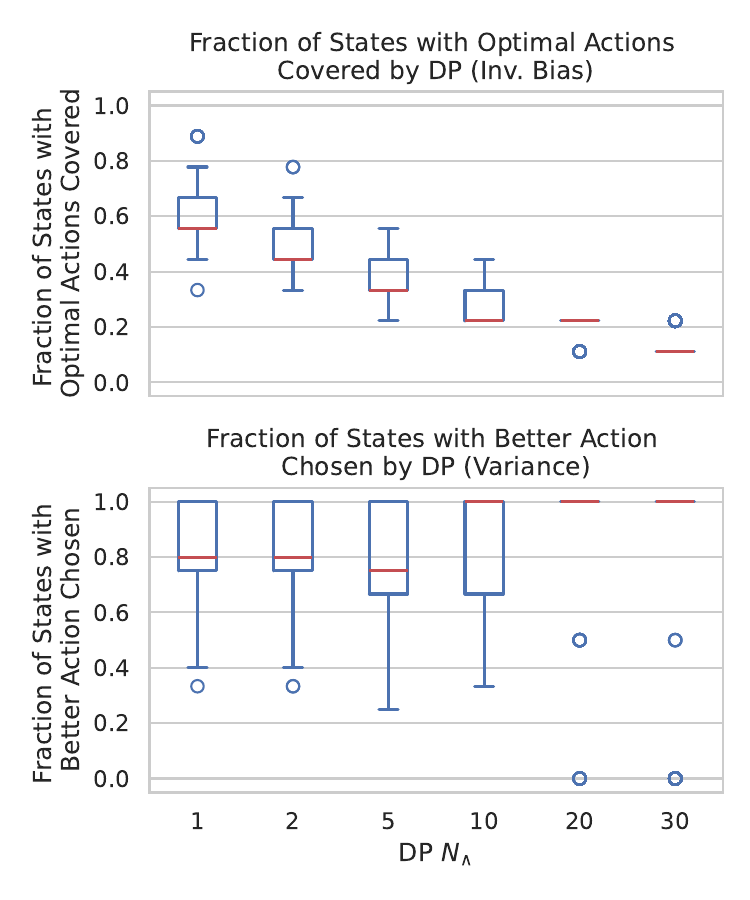}
    \begin{minipage}[b]{.33\textwidth}
        \centering
        \vspace{-10pt}
        \includegraphics[width=.9\textwidth, trim=0 19 0 0, clip]{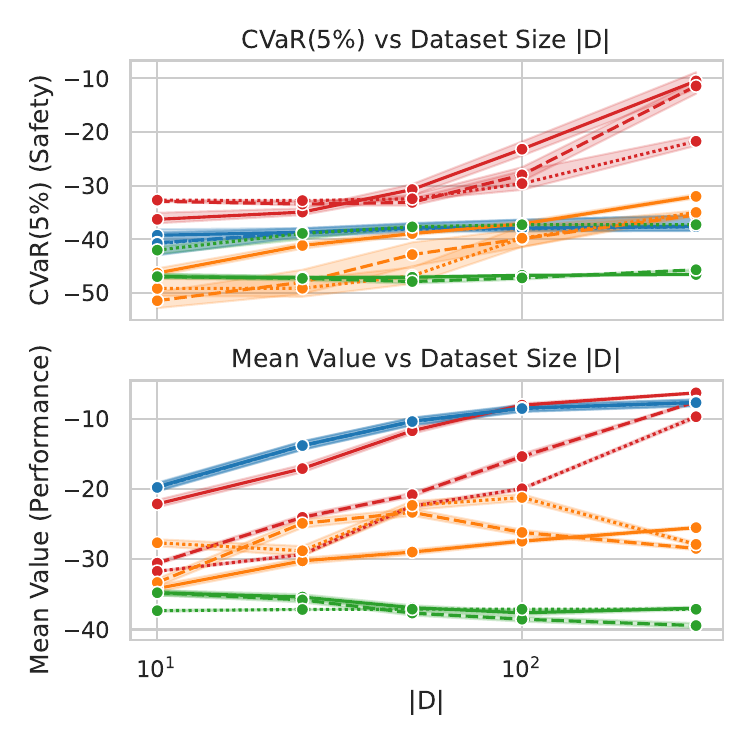}        \includegraphics[width=.7\textwidth]{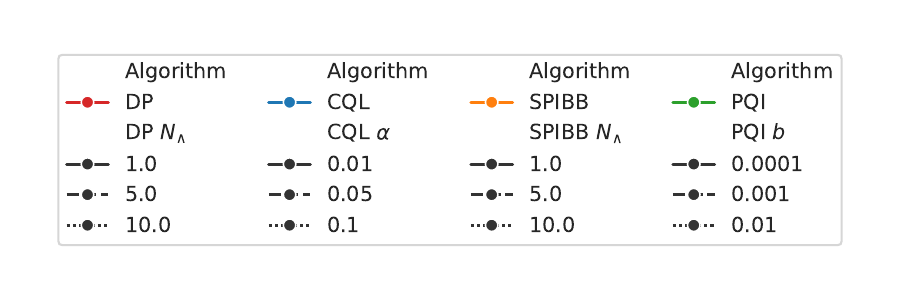}
    \end{minipage}    
    \caption{\textbf{GridWorld}: 
    (left) Illustration of our Gridworld environment,
    (middle) Bias-variance trade-off managed by $\Nmin$,
    (right) Performance of DPRL in terms of CVaR and Mean Value. DPRL provides safe policy improvement (CVaR), while matching baselines on mean value.}
    \label{fig:gw_results}
\end{figure*}

\textbf{Discrete Case.}
Define the following sets constructed from the dataset $\gD$:
\begin{gather}
    \label{eq:advantageous_actions_discrete}
    \gAdp_s = \set{ a \in \gA : n(s, a) \geq \Nmin \text{ and } \hat{Q}^{\pi_b}(s, a) \geq \hat{V}^{\pi_b}(s) }, \nonumber \\\quad
    \gSdp = \set{ s \in \gS : \gAdp_s \neq \emptyset }
\end{gather}
where $\hat{Q}(s, a)$ and $\hat{V}(s)$ are the estimated values under the behavior policy $\pib$.  The set $\gAdp = \set{\gAdp_s}$ represents the set actions in each state that we are confident are advantageous over the behavior policy (in the sense that the number of visits $n(s, a) \geq \Nmin$).  The set $\gSdp$ is the set of states where at least one advantageous action exists. These states are the \emph{decision points} on which we will  recommend changes.  We will defer to the behavior policy in the set $\Phi = \set{s \in \gS : \gAdp_s = \emptyset}$ because cannot be confident that an advantageous action exists in those states.

Once the decision points are determined, we create an ``elevated'' Semi-MDP (SMDP) $\Mdp = (\gSdp, \gA, \Pdp, \Rdp, \gammadp)$, where the state space is restricted to the decision points $\gSdp$, and the transition, reward, discount functions are estimated using the dataset $\gD$.
We use the SMDP framework since the number of time steps to reach the next decision point is not fixed, and can vary with each $(s, a, s')$ triplet.
The policy set $\Pi^{\text{DP}}$ is defined as the set of deterministic policies that select can only select advantageous actions in each state:
\begin{gather}
    \Pi^{\text{DP}} = \set{ \pi : \pi(a|s) = 0 \quad \forall a \notin \gAdp_s }
\end{gather}
For each state $s \in \gSdp$, action $a \in \gAdp_s$, and next state $s' \in \gSdp$, we say that a transition $(s, a, s', k)$ exists in a trajectory if $s'$ is the first decision point in the trajectory starting from $(s, a)$ and taking $k$ steps (only considering first visits), and $\rdp(n, s, a, s', k)$ is the discounted sum of rewards in the trajectory $n$ starting from $(s, a)$ and ending in $s'$ after $k$ steps, and $0$ if no such transition exists.
Then, the counts $\tilde{n}(s, a, s', k)$ are defined as the number of times $(s, a, s', k)$ exists in the dataset, and $\tilde{P}(s' | s, a)$, $\gammadp(s, a, s')$ and $\Rdp(s, a)$ are computed from these counts (see Algorithm~\ref{alg:make_smdp} in the appendix for details):
\begin{gather}
    \tilde{P}(s' | s, a) = \frac{\sum_{k=1}^{T} \ndp(s, a, s', k)}{\sum_{s' \in \gSdp} \sum_{k=1}^{T} \ndp(s, a, s', k)}, \quad \\
    \gammadp(s, a, s') = \frac{\sum_{k=1}^{T} \ndp(s, a, s', k) \gamma^k}{\sum_{k=1}^{T} \ndp(s, a, s', k)}, \notag\\
    \Rdp(s, a) = \frac{\sum_{s' \in \gSdp} \sum_{k=1}^{T} \sum_{n=1}^{N} \rdp(n, s, a, s', k)}{\sum_{s' \in \gSdp} \sum_{k=1}^{T} \ndp(s, a, s', k)}
\end{gather}

We optimize over the policy set $\Pi^{\text{DP}}$ by using policy iteration \citep{bradtkeReinforcementLearningMethods1994,suttonMDPsSemiMDPsLearning} on $\Mdp$. In each iteration $i$, the policy $\pi^{(i)}$ is evaluated to get $V^{(i)}_{\Mdp}$, which is then improved to get $\pi^{(i+1)}$:
{\small
\begin{gather}
\label{eq:policy_improvement}
    \pi^{(i+1)}(s) = \argmax_{a \in \gAdp_s} \Rdp(s, a) + \E_{\Pdp(s' | s, a)} \gammadp(s, a, s') V^{(i)}_{\Mdp}(s')
\end{gather}
}

If the policy iteration converges after $K$ steps to give $\pi^{(K)}$, the final policy $\pidp$ can then be used to either defer to the behavior policy or to make a better decision for a given state $s$:
\begin{gather}
    \pidp(s) = \begin{cases}
        \text{DEFER} & \text{if } s \in \Phi \\
        % \pib(s) & \text{if } s \in \Phi \\
        \pi^{(K)}(s) & \text{otherwise}
    \end{cases}
\end{gather}

\textbf{Continuous Case.}
For continuous state spaces, we describe a variant of the DP algorithm that can be used to provide safe policy improvements for any given state (for pseudocode, see Algorithm~\ref{alg:dp_continuous} in the appendix). The algorithm involves storing the dataset $\gD$.
We define the following distance metric over state-action pairs $(s, a)$ and states $s$:
\begin{gather*}
    d((s, a), (s', a')) = d(s, s') \text{ if } a = a', \text{ and } \infty \text{ otherwise} \\
    d(s, s') = \norm{s - s'}
\end{gather*}
Given a state $s$, define the set of neighbors of $s$ within a ball of radius $r$ as $\gN(s) = \set{ S^n_t : d(s, S^n_t) \leq r \text{ for } n = 1, \ldots, N }$ and the count of neighbors as $n(s) = \abs{\gN(s)}$.
Note that our analysis assumes the first neighboring state is included per trajectory $n$. However, in practice, including all neighbors helps reduce the variance of the estimates.
Further, for each action $a$, $\gN(s, a) = \set{ (S^n_t, A^n_t) : d((s, a), (S^n_t, A^n_t)) \leq r \text{ for } n = 1, \ldots, N }$ and $n(s, a) = \abs{\gN(s, a)}$.
We use a Ball-Tree data structure to efficiently find a neighbor in $\gO(\log K)$ time after $\gO(K \log K)$ preprocessing time, where $K$ is the total number of points to search over.
Analogous to the discrete case, we define the set of advantageous actions in state $s$ as:

{\small
\begin{gather}
    \label{eq:advantageous_actions_continuous}
    \gAdp_s = \set{ a \in \gA : n(s, a) \geq \Nmin \text{ and } \hat{Q}^{\pib}(s, a) \geq \hat{V}^{\pib}(s) }
\end{gather}}

where $\hat{Q}^{\pib}(s, a)$ and $\hat{V}^{\pib}(s)$ are estimated by averaging the returns of the neighbors in $\gN(s, a)$ and $\gN(s)$, respectively.
We defer to the behavior if no advantageous actions exist in the state, and return the action with the highest estimated $Q$ value otherwise:
{\small
\begin{gather}
    \pidp(s) = \begin{cases}
        \text{DEFER} & \text{if } \gAdp_s = \emptyset \\
        % \pib(s) & \text{if } \gAdp_s = \emptyset \\
        \argmax_{a \in \gAdp_s} \hat{Q}^{\pib}(s, a) & \text{otherwise}
    \end{cases}
\end{gather}}

The policy is implicitly defined using the dataset $\gD$, and we compute it for any state.
The above procedure uses non-parametric estimation of $Q^{\pib}$- and $V^{\pib}$-values using the dataset $\gD$.
The $r$ hyperparameter controls the trade-off between bias and variance in the estimates, as well as the sparsity of improvements.
A smaller $r$ will achieve lower bias and fewer possible decision points, while a larger $r$ will result in a higher bias of estimates and more decision points.
Note that the variance of the estimates is still controlled because we average over at least $\Nmin$ neighbors in the estimate.
The trade-off versus a parametric approach is that we do not need to assume a particular form of the value function, at the expense of storing the dataset $\gD$.
This can be a reasonable trade-off in practice when the dataset is not too large.
We show in the analysis section that this non-parametric approach can achieve tighter bounds than a parametric approach.

\section{Analysis}
\label{sec:analysis}
We now provide performance bounds for the algorithms DPRL-D and DPRL-C presented in previous section.
The following theorem proves that $\pidp$ is a safe policy improvement over the behavior $\pib$. 
\begin{restatable}[DPRL Discrete]{theorem}{thmdpdiscrete}
    \label{thm:dp_discrete}
    Let $\pidp$ be the policy obtained by the DP algorithm. Then $\pidp$ is a safe policy improvement over the behavior policy $\pib$, with probability at least $1-\delta$
    \begin{gather}
        \rho(\pidp) - \rho(\pib)  \geq - \frac{\Vmax}{1-\gamma} \sqrt{\frac{1}{ \Nmin } \log \frac{C(\Nmin)}{\delta}}
    \end{gather}
    where $C(\Nmin)$ is the count of the number of $(s,a)$ pairs that are observed at least $\Nmin$ times in the dataset:
    \begin{gather}
        C(\Nmin) = \sum_{s \in \gS} \sum_{a \in \gA} \indicator \brackets{n(s, a) \geq \Nmin}
    \end{gather}
\end{restatable}

% \begin{restatable}[DPRL Discrete]{theorem}{thmdpdiscrete}
%     \label{thm:dp_discrete}
%     Let $\pidp$ be the policy obtained by the DP algorithm. Then $\pidp$ is a safe policy improvement over the behavior policy $\pib$, with probability at least $1-\delta$
%     \begin{gather}
%         V^{(\pidp)} - V^{(\pib)}  \geq - \frac{\Vmax}{1-\gamma} \sqrt{\frac{1}{ \Nmin } \log \frac{C(\Nmin)}{\delta}}
%     \end{gather}
%     where $C(\Nmin)$ is the count of the number of $(s,a)$ pairs that are observed at least $\Nmin$ times in the dataset:
%     \begin{gather}
%         C(\Nmin) = \sum_{s \in \gS} \sum_{a \in \gA} \indicator \brackets{n(s, a) \geq \Nmin}
%     \end{gather}
% \end{restatable}

\textbf{Proof sketch.}
The key fact we exploit is the property that $\pidp$ always takes an advantageous action with respect to $\hat{Q}^\pib - \hat{V}^\pib$.
We express the $\hat{Q}^\pib$ and $\hat{V}^\pib$ as the first-visit monte-carlo average of the observed returns, making sure to split the returns into independent random variables (since the returns used to estimate the $Q$-values and $V$-values may overlap).
We then bound the advantage by using the fact that the advantage is zero for $(s,a)$ pairs for the states we defer, and bounded for the $(s,a)$ pairs that are observed at least $\Nmin$ times.  (See Appendix~\ref{sec:proofs} for the full proof.)

\textbf{Discussion.}
Our bound is a function of hyperparameters $\Nmin$ and $\delta$, and the data-dependent term $C(\Nmin)$.
The $C(\Nmin)$ term in our bound is a count of the number of $(s,a)$ pairs that are observed at least $\Nmin$ times in the dataset.
This term is much smaller than $\abs{\gS} \abs{\gA}$ when the behavior policy visits only a small subset of the state-action space.
The hyperparameter $\Nmin$ allows us to directly control the trade-off between high-confidence policy improvement (high $\Nmin$) and higher performance improvement at the cost of safety (low $\Nmin$).

\textbf{Comparison to baselines.}
We summarize the main differences between our bound and prior work here.  The bounds in prior work are reproduced in Appendix~\ref{sec:prior_bounds} for reference.  Most importantly, our dependence on $\abs{\gS}$ and $\abs{\gA}$ comes indirectly through the $C(\Nmin)$ term, which differs from the direct dependence on $\abs{\gS}$ and $\abs{\gA}$ by SPI and pessimism-based methods \citep{liuProvablyGoodBatch2020c,kimModelbasedOfflineReinforcement2023b}.
Thus, our bound is much tighter when the behavior policy has only visited a small subset of the state-action space more than $\Nmin$ times: we scale in the number of visited parts of the state-action space $C(\Nmin)$ rather than $\abs{\gS} \abs{\gA}$.  This difference would be most present when the behavior policy and transition dynamics are close to deterministic, or when the size of dataset is small relative to the size of the state-action space.
On the other hand, when the size of the dataset is large and the behavior distribution is closer to uniform, all bounds will be tight.
Our dependence on the $\Nmin$ parameter and effective horizon $1/(1-\gamma)$ matches the SPI \citep{larocheSafePolicyImprovement2019,schollSafePolicyImprovement2022,wienhoftMoreLessSafe2023a} literature.  However, that we do not require access to $\pi_b$. Unlike Corollary 2 of \citep{liuProvablyGoodBatch2020c}, our bound does not have a dependence on the threshold $b$ of the state-action density $\hat{\mu}(s,a)$, which has a direct correspondence to our $\Nmin$ parameter as $b = \Nmin/|D|$.
As a result, their bound gets looser as $b$ gets smaller. This happens when the size of the dataset gets larger while keeping the set of $(s,a)$ pairs with $\Nmin$ observations unchanged (i.e., more trajectories were added in the low-density regions). Our DPRL bound is unaffected by this superfluous extra data.

% \subsection{Corollary: When the optimal policy is sufficiently covered}
% {\color{red} This is mostly going to be dropped because I can't find a way to prove it.}
% \begin{corollary}  
% \label{cor:dp_discrete_optimal}
%     If the optimal deterministic policy $\pi^*$ is sufficiently covered by the dataset, i.e., $\forall s \in \gS$, either $n(s, \pi^*(s)) \geq \Nmin$ or $\pi^*(s) = \pib(s)$, then we have the following bound:
%    \begin{gather}
%         \rho(\pi^*) - \rho(\pidp) \leq  ???
%     \end{gather}
% \end{corollary}

\textbf{Continuous case: DPRL-C.}
The following theorem proves that $\pidp$ is a safe policy improvement over the behavior $\pib$ in the continuous case.
Define $M(r, \Nmin)$ as the number of balls of radius $r$ needed to cover the subset of $\gX \subset \gS \times \gA$ where each $(s,a) \in \gX$ has at least $\Nmin$ data points in the ball $B_r(s,a)$. Similarly, define $M(r)$ as the number of balls of radius $r$ needed to cover the entire state-action space.
Also define $\epsilon_r$ as the maximum error in the $Q$-values in the ball $B_r(s,a)$. Finally, we assume that the error in estimating the action-values is bounded for all points $(s',a') \in B_r(s,a)$: $\abs{Q(s,a) - Q(s',a')} \leq \epsilon_r$

% \begin{theorem}[DP Continuous]
\begin{restatable}[DPRL Continuous]{theorem}{thmdpcontinuous}
    \label{thm:dp_continuous}
    Let the constant $M(r, \Nmin)$ be a measure of the volume on $\gS \times \gA$ that the dataset $\gD$ covers, and let the error in estimating the $Q$-values using a neighbor is bounded by $\epsilon_r$, then $\pidp$ is a safe policy improvement over the behavior policy $\pib$. That is, with probability at least $1-\delta$:
    {\small
    \begin{gather*}
        \rho(\pidp) - \rho(\pib)  \geq -\frac{\Vmax}{1-\gamma} \sqrt{\frac{1}{ 2 \Nmin } \log \frac{M(r, \Nmin)}{\delta}} - 3 \epsilon_r
    \end{gather*}}
% \end{theorem}
\end{restatable}

\textbf{Discussion.}
The $M(r, \Nmin)$ term is a measure of the size of the region in the state-action space for which the policy improvement is guaranteed.
For datasets with density only on a small subset of the state-action space, the $M(r, \Nmin)$ term will be much smaller than $M(r)$.
The second term in the bound quantifies the penalty paid for using neighborhood-based estimates of the $Q$-values.
The hyperparameter $\Nmin$ influences the bound directly and indirectly through the $M(r, \Nmin)$ term. Higher $\Nmin$ leads to a high-confidence estimation of the $Q$-values, and also leads to a smaller $M(r, \Nmin)$ term since fewer $(s,a)$ pairs can be improved upon.
The hyperparameter $r$ influences the bound through the $M(r, \Nmin)$ term and the $\epsilon_r$ term. Smaller $r$ leads to a low-bias estimate of the $Q$-values, but also leads to a larger $M(r, \Nmin)$ term.

\textbf{Comparison to baselines.}
DeepAveragers \citep{shresthaDeepAveragersOfflineReinforcement2020} (which uses kNN to identify neighbors) also have a dependence on a term similar to $M(r, \Nmin)$ in our bound. However, their term approaches $M(r)$ if the dataset is sparse in the sense that each $(s,a)$ pair has very few neighborhood points in the dataset (for DPRL, the $M(r, \Nmin)$ will be close $0$ in this case).
The $\log M(r, \Nmin)$ term in our bound is a non-parametric alternative to the $\log \abs{\gF}$ term in bounds for methods which use a parametric function class $\gF$ (e.g., \citep{liuProvablyGoodBatch2020c})
The inherent trade-off is that the parametric methods optimize for a \emph{global} estimation error $\epsilon_{\gF}$ over the dataset (and hence have a dependence on terms such as $\epsilon_{\gF}$, $\abs{\gF}$ and $\abs{\gD}$), while our non-parametric method optimizes for a \emph{local} estimation error over the neighborhood of each $(s,a)$ pair (and hence has a dependence on $\epsilon_r$, $M(r, \Nmin)$ and $\Nmin$).
For large datasets with uniform exploration, both parametric and non-parametric methods can have a low generalization error over the dataset, and thus the improvement is similar.
However, for small datasets or for datasets with non-uniform exploration, parametric methods can have a looser bound because $\epsilon_{\gF}$ can be large, while our non-parametric method can have a tighter bound because $M(r, \Nmin)$ can be much smaller because we only incur the errors in dense regions of the dataset.
Finally, we note that our bound is more actionable than the bounds in prior work because we can estimate $M(r, \Nmin)$ using DBSCAN \citep{ester1996density} (by computing the total volume occupied by the core points).

\begin{figure}
    \centering
    \includegraphics[width=.6\linewidth]{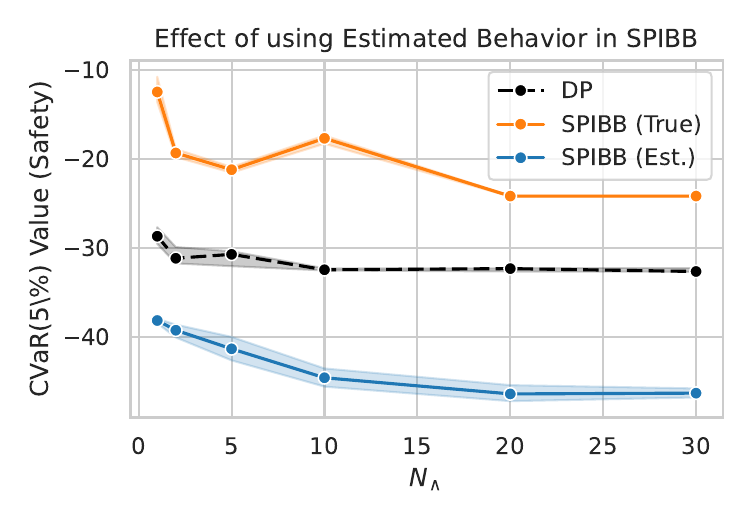}
    \caption{Gridworld: Replacing true behavior with estimated behavior in SPIBB leads to degraded CVaR. DP performs better than SPIBB without access to the true behavior.}
    \label{fig:spibb_with_truebehavior}
\end{figure}

\begin{figure}
  \centering
  \includegraphics[width=0.45\linewidth]{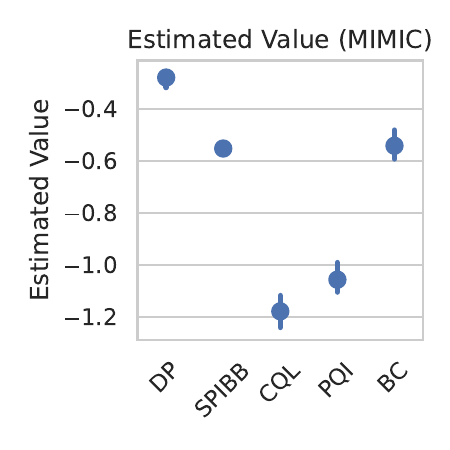}
  \includegraphics[width=0.45\linewidth]{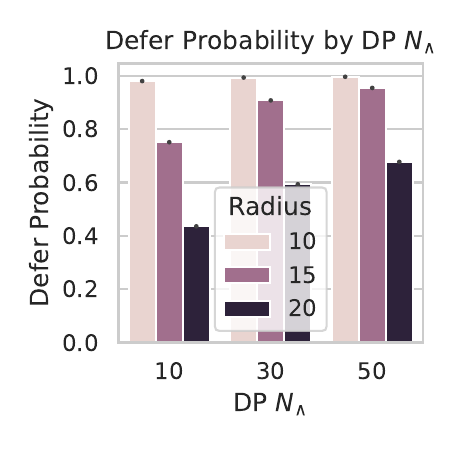}
  \caption{(Left) OPE value estimates of the learned policies on the MIMIC dataset: DPRL achieves the highest estimated value. (Right) The reason for DPRL's strong performance: for the chosen hyperparameters ($\Nmin = 50$, $r=10$), DPRL defers to the behavior in nearly all states except where it is confident it can achieve a better outcome.
  \label{fig:mimic_results}
  }
\end{figure}

\begin{figure}
    \centering
    \includegraphics[width=\linewidth]{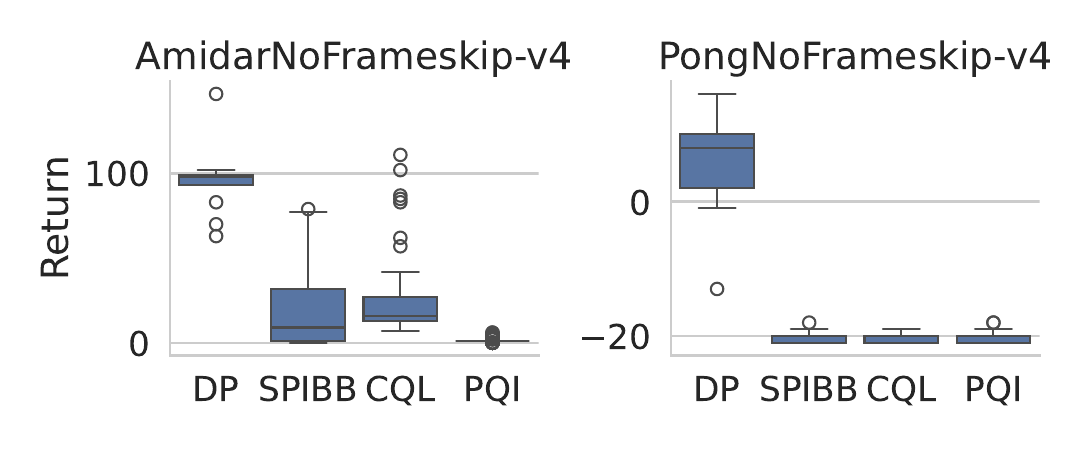}
    \caption{DP consistently learns good policies from suboptimal behavior data across Atari environments. Each algorithm is trained on 100,000 samples and evaluated on 20 episodes after training.}
    \label{fig:atari_results}
\end{figure}

\section{Experimental Evaluation}
\label{sec:experiments}

We first evaluate the DPRL on datasets from discrete-state-action MDPs and GridWorld and from continuous-state, discrete-action Atari environments, where we can use the simulators to accurately estimate performance. Finally, we use our algorithm on continuous-state, discrete-action-real-world dataset of hypotensive patients in the ICU.
We compare DPRL to SPIBB \citep{larocheSafePolicyImprovement2019}, PQI \citep{liuProvablyGoodBatch2020c}, and CQL \citep{kumarConservativeQLearningOffline2020}, which are typical batch RL methods used for safe policy improvement and span the density-based, pessimism-based, and support-based constraints.

% We evaluate the performance of DPRL and the baselines on a discrete-state  and a real EHR data from the MIMIC-IV care database.

In what follows, we briefly describe the domain and then discuss key results from our experiments. Specifically, we consider results over discrete state and action spaces using Toy MDP and Gridworld as domains; as well as continuous state and discrete action spaces based on Atari and MIMIC III.

\subsection*{Discrete State and Action Spaces: MDP and GridWorld}

We describe the MDPs in Section~\ref{sec:mdp_example}. For both MDPs, we vary the number of trajectories in the datasets. For our $10 \times 10$ GridWorld, we create datasets by simulating a `careless' expert which is optimal everywhere except in a few states, where it chooses the worst action with probability 0.9 (marked orange in Figure~\ref{fig:gw_results}). We sample datasets of $n$ trajectories, $n \in \set{10, 25, 50, 100, 500}$.

\textbf{DPRL achieves better theoretical bounds.}
On the MDP and GridWorld (Figures~\ref{fig:mdp_example} and\ref{fig:gw_results}, we see that the DPRL provides safe policy improvement (CVaR).
In GridWorld plots, we see that the CVaR without mean value being affected.
The $\Nmin$ parameter allows us to control the trade-off between performance and safety, as shown in Figure~\ref{fig:gw_results}(right): as $\Nmin$ increases, the mean value (performance) decreases, but the safety increases.
On both the MDP and GridWorld, we see that the DPRL provides tighter bounds compared to the baselines, as shown in Figure~\ref{fig:mdp_example}(center).
The bounds are tighter because the $C(\Nmin)$ term in our bound is much smaller than the $\abs{\gS} \abs{\gA}$ term in the SPIBB and PQI bounds. Also, since the bound is data-dependent, it does not degrade as quickly as SPIBB and PQI for increase in $\abs{\gS}$.

\textbf{DPRL better manages the bias/variance tradeoff than existing methods.}
Figure~\ref{fig:gw_results}(center) provides an insight into how  $\Nmin$ achieves the trade-off between safety and performance. As $\Nmin$ increases, the fraction of states where the optimal action is allowed decreases, leading to increase in bias. At the same time, the fraction of states where a better action is chosen increases (considering states where multiple actions are allowed). This suggests reduced variance in value estimation.
\subsection*{Atari and Hypotension Datasets}
\textbf{DPRL achieves good performance in high-dimensional settings.} 
On Atari domains, in Figure \ref{fig:atari_results}, we observe that DPRL achieves good performance even with 100K samples and a `careless' expert who takes the worst action $50\%$ of the time.
In the hypotension dataset with continuous states, DPRL is the only method that achieves higher OPE estimates than the behavior baseline (Figure~\ref{fig:mimic_results}). We also see how the ability to defer can be useful in such a complex task: for the chosen parameters $\Nmin=50, r=10$, DPRL defers more than $95\%$ of the time!
Such a policy is also actionable, since the physician can review the actions much faster than with the other methods.
\section{Limitations and Discussion}
The DPRL approach we present is non-parametric and requires storing the entire training data which can be costly. While we used BallTrees for efficient neighbor search, it would be interesting to explore how compression methods (e.g., coresets) can be used to store only a subset of the data. 
% or by learning rules to identify high density regions of the space and only keeping returns in that (s,a)-space.
DPRL-C uses the euclidean distance with continuous states. Future work can look at more sophisticated metrics (e.g. bisimulation distance) and explore their influence on the safety bounds.
Finally, we developed the SMDP formulation for the DPRL-D algorithm which performed multi-step planning but the DPRL-C algorithm only performed 1-step planning over behavior. In the future, we plan to extend the DPRL-C algorithm to multi-step planning while still maintaining safety guarantees.

\section{Conclusion and Future Work}
We introduced a decision points-based RL algorithm for performing safe policy improvement with guarantees. Our approach explicitly constrains the set of state-action pairs or regions of states considered to those areas that are most densely visited, while leveraging data from sparsely visited states to determine the ways in which we may deviate from the behaviour policy. Our experiments demonstrated that our approach lead to safer policy improvement with more confident estimates and tighter guarantees on policy improvement. Future work could explore how to extend the DPRL-C algorithm to multi-step planning while still maintain safety guarantees and utilize embeddings which are more sophisticated than euclidean distance and to determine their influence on the safety bounds.

\bibliography{ref}

\begin{thebibliography}{43}
\providecommand{\natexlab}[1]{#1}
\providecommand{\url}[1]{\texttt{#1}}
\expandafter\ifx\csname urlstyle\endcsname\relax
  \providecommand{\doi}[1]{doi: #1}\else
  \providecommand{\doi}{doi: \begingroup \urlstyle{rm}\Url}\fi

\bibitem[Agarwal et~al.(200)Agarwal, Jiang, Kakade, and Sun]{agarwalReinforcementLearningTheorya}
Alekh Agarwal, Nan Jiang, Sham~M Kakade, and Wen Sun.
\newblock \emph{Reinforcement {Learning}: {Theory} and {Algorithms}}.
\newblock 200.

\bibitem[Bradtke and Duff(1994)]{bradtkeReinforcementLearningMethods1994}
Steven Bradtke and Michael Duff.
\newblock Reinforcement {Learning} {Methods} for {Continuous}-{Time} {Markov} {Decision} {Problems}.
\newblock In \emph{Advances in {Neural} {Information} {Processing} {Systems}}, volume~7. MIT Press, 1994.
\newblock URL \url{https://proceedings.neurips.cc/paper/1994/hash/07871915a8107172b3b5dc15a6574ad3-Abstract.html}.

\bibitem[Char et~al.(2022)Char, Mehta, Villaflor, Dolan, and Schneider]{charBATSBestAction2022}
Ian Char, Viraj Mehta, Adam Villaflor, John~M. Dolan, and Jeff Schneider.
\newblock {BATS}: {Best} {Action} {Trajectory} {Stitching}, April 2022.
\newblock URL \url{http://arxiv.org/abs/2204.12026}.
\newblock arXiv:2204.12026 [cs].

\bibitem[Ester et~al.(1996)Ester, Kriegel, Sander, Xu, et~al.]{ester1996density}
Martin Ester, Hans-Peter Kriegel, Jorg Sander, Xiaowei Xu, et~al.
\newblock A density-based algorithm for discovering clusters in large spatial databases with noise.
\newblock In \emph{kdd}, volume~96, pages 226--231, 1996.

\bibitem[Fu et~al.()Fu, Di, and Boulet]{fuBatchReinforcementLearning2020}
Yuwei Fu, Wu~Di, and Benoit Boulet.
\newblock Batch reinforcement learning in the real world: {{A}} survey.
\newblock In \emph{Offline {{RL Workshop}}, {{NeurIPS}}}.

\bibitem[Fujimoto et~al.(2019)Fujimoto, Meger, and Precup]{fujimotoOffPolicyDeepReinforcement2019a}
Scott Fujimoto, David Meger, and Doina Precup.
\newblock Off-{Policy} {Deep} {Reinforcement} {Learning} without {Exploration}.
\newblock In \emph{Proceedings of the 36th {International} {Conference} on {Machine} {Learning}}, pages 2052--2062. PMLR, May 2019.
\newblock URL \url{https://proceedings.mlr.press/v97/fujimoto19a.html}.
\newblock ISSN: 2640-3498.

\bibitem[Garc{\i}a and Fern{\'a}ndez(2015)]{garcia2015comprehensive}
Javier Garc{\i}a and Fernando Fern{\'a}ndez.
\newblock A comprehensive survey on safe reinforcement learning.
\newblock \emph{Journal of Machine Learning Research}, 16\penalty0 (1):\penalty0 1437--1480, 2015.

\bibitem[Ghavamzadeh et~al.(2016)Ghavamzadeh, Petrik, and Chow]{ghavamzadehSafePolicyImprovement2016}
Mohammad Ghavamzadeh, Marek Petrik, and Yinlam Chow.
\newblock Safe {Policy} {Improvement} by {Minimizing} {Robust} {Baseline} {Regret}.
\newblock In \emph{Advances in {Neural} {Information} {Processing} {Systems}}, volume~29. Curran Associates, Inc., 2016.
\newblock URL \url{https://proceedings.neurips.cc/paper/2016/hash/9a3d458322d70046f63dfd8b0153ece4-Abstract.html}.

\bibitem[Gottesman et~al.(2019)Gottesman, Liu, Sussex, Brunskill, and Doshi-Velez]{gottesmanCombiningParametricNonparametric2019}
Omer Gottesman, Yao Liu, Scott Sussex, Emma Brunskill, and Finale Doshi-Velez.
\newblock Combining parametric and nonparametric models for off-policy evaluation.
\newblock In \emph{Proceedings of the 36th {International} {Conference} on {Machine} {Learning}}, pages 2366--2375. PMLR, May 2019.
\newblock URL \url{https://proceedings.mlr.press/v97/gottesman19a.html}.
\newblock ISSN: 2640-3498.

\bibitem[Gottesman et~al.(2020)Gottesman, Futoma, Liu, Parbhoo, Celi, Brunskill, and Doshi-Velez]{gottesman_interpretable_2020}
Omer Gottesman, Joseph Futoma, Yao Liu, Sonali Parbhoo, Leo~Anthony Celi, Emma Brunskill, and Finale Doshi-Velez.
\newblock Interpretable {Off}-{Policy} {Evaluation} in {Reinforcement} {Learning} by {Highlighting} {Influential} {Transitions}, August 2020.
\newblock URL \url{http://arxiv.org/abs/2002.03478}.
\newblock arXiv:2002.03478 [cs, stat].

\bibitem[Hajek and Raginsky(2019)]{hajekECE543Statistical2019}
Bruce Hajek and Maxim Raginsky.
\newblock {ECE} 543: {Statistical} {Learning} {Theory}, 2019.
\newblock URL \url{https://courses.engr.illinois.edu/ece543/sp2019/SLT.pdf}.

\bibitem[Hepburn and Montana(2022)]{hepburnModelbasedTrajectoryStitching2022}
Charles~A. Hepburn and Giovanni Montana.
\newblock Model-based {Trajectory} {Stitching} for {Improved} {Offline} {Reinforcement} {Learning}, November 2022.
\newblock URL \url{http://arxiv.org/abs/2211.11603}.
\newblock arXiv:2211.11603 [cs, stat].

\bibitem[Jiang and Li()]{jiangDoublyRobustPolicy}
Nan Jiang and Lihong Li.
\newblock Doubly {{Robust Off-policy Value Evaluation}} for {{Reinforcement Learning}}.

\bibitem[Johnson et~al.(2023)Johnson, Bulgarelli, Shen, Gayles, Shammout, Horng, Pollard, Hao, Moody, Gow, et~al.]{johnson2023mimic}
Alistair~EW Johnson, Lucas Bulgarelli, Lu~Shen, Alvin Gayles, Ayad Shammout, Steven Horng, Tom~J Pollard, Sicheng Hao, Benjamin Moody, Brian Gow, et~al.
\newblock Mimic-iv, a freely accessible electronic health record dataset.
\newblock \emph{Scientific data}, 10\penalty0 (1):\penalty0 1, 2023.

\bibitem[Joshi et~al.(2022)Joshi, Parbhoo, and Doshi-Velez]{joshiLearningtodeferSequentialMedical2022}
Shalmali Joshi, Sonali Parbhoo, and Finale Doshi-Velez.
\newblock Learning-to-defer for sequential medical decision-making under uncertainty.
\newblock \emph{Transactions on Machine Learning Research}, December 2022.
\newblock ISSN 2835-8856.
\newblock URL \url{https://openreview.net/forum?id=0pn3KnbH5F}.

\bibitem[Kalashnikov et~al.(2018)Kalashnikov, Irpan, Pastor, Ibarz, Herzog, Jang, Quillen, Holly, Kalakrishnan, Vanhoucke, et~al.]{kalashnikov2018scalable}
Dmitry Kalashnikov, Alex Irpan, Peter Pastor, Julian Ibarz, Alexander Herzog, Eric Jang, Deirdre Quillen, Ethan Holly, Mrinal Kalakrishnan, Vincent Vanhoucke, et~al.
\newblock Scalable deep reinforcement learning for vision-based robotic manipulation.
\newblock In \emph{Conference on robot learning}, pages 651--673. PMLR, 2018.

\bibitem[Kidambi et~al.(2021)Kidambi, Rajeswaran, Netrapalli, and Joachims]{kidambiMOReLModelBasedOffline2021}
Rahul Kidambi, Aravind Rajeswaran, Praneeth Netrapalli, and Thorsten Joachims.
\newblock {MOReL} : {Model}-{Based} {Offline} {Reinforcement} {Learning}, March 2021.
\newblock URL \url{http://arxiv.org/abs/2005.05951}.
\newblock arXiv:2005.05951 [cs, stat].

\bibitem[Kim and Oh(2023)]{kimModelbasedOfflineReinforcement2023b}
Byeongchan Kim and Min-Hwan Oh.
\newblock Model-based {Offline} {Reinforcement} {Learning} with {Count}-based {Conservatism}.
\newblock In \emph{Proceedings of the 40th {International} {Conference} on {Machine} {Learning}}, pages 16728--16746. PMLR, July 2023.
\newblock URL \url{https://proceedings.mlr.press/v202/kim23q.html}.
\newblock ISSN: 2640-3498.

\bibitem[Kumar et~al.(2019)Kumar, Fu, Soh, Tucker, and Levine]{kumarStabilizingOffPolicyQLearning2019a}
Aviral Kumar, Justin Fu, Matthew Soh, George Tucker, and Sergey Levine.
\newblock Stabilizing off-policy q-learning via bootstrapping error reduction.
\newblock \emph{Advances in neural information processing systems}, 32, 2019.

\bibitem[Kumar et~al.(2020)Kumar, Zhou, Tucker, and Levine]{kumarConservativeQLearningOffline2020}
Aviral Kumar, Aurick Zhou, George Tucker, and Sergey Levine.
\newblock Conservative {Q}-{Learning} for {Offline} {Reinforcement} {Learning}.
\newblock In \emph{Advances in {Neural} {Information} {Processing} {Systems}}, volume~33, pages 1179--1191. Curran Associates, Inc., 2020.
\newblock URL \url{https://proceedings.neurips.cc/paper/2020/hash/0d2b2061826a5df3221116a5085a6052-Abstract.html}.

\bibitem[Kumar et~al.(2022)Kumar, Hong, Singh, and Levine]{kumarWhenShouldWe2022}
Aviral Kumar, Joey Hong, Anikait Singh, and Sergey Levine.
\newblock When {Should} {We} {Prefer} {Offline} {Reinforcement} {Learning} {Over} {Behavioral} {Cloning}?, April 2022.
\newblock URL \url{http://arxiv.org/abs/2204.05618}.
\newblock arXiv:2204.05618 [cs].

\bibitem[Lange et~al.(2012)Lange, Gabel, and Riedmiller]{lange2012batch}
Sascha Lange, Thomas Gabel, and Martin Riedmiller.
\newblock Batch reinforcement learning.
\newblock In \emph{Reinforcement learning: State-of-the-art}, pages 45--73. Springer, 2012.

\bibitem[Laroche et~al.(2019)Laroche, Trichelair, and Combes]{larocheSafePolicyImprovement2019}
Romain Laroche, Paul Trichelair, and Rémi Tachet~des Combes.
\newblock Safe {Policy} {Improvement} with {Baseline} {Bootstrapping}, June 2019.
\newblock URL \url{http://arxiv.org/abs/1712.06924}.
\newblock arXiv:1712.06924 [cs, stat] version: 5.

\bibitem[Lederer et~al.(2019)Lederer, Umlauft, and Hirche]{ledererUniformErrorBounds2019}
Armin Lederer, Jonas Umlauft, and Sandra Hirche.
\newblock Uniform {Error} {Bounds} for {Gaussian} {Process} {Regression} with {Application} to {Safe} {Control}.
\newblock In \emph{Advances in {Neural} {Information} {Processing} {Systems}}, volume~32. Curran Associates, Inc., 2019.
\newblock URL \url{https://proceedings.neurips.cc/paper_files/paper/2019/hash/fe73f687e5bc5280214e0486b273a5f9-Abstract.html}.

\bibitem[Li et~al.(2011)Li, Littman, Walsh, and Strehl]{liKnowsWhatIt2011}
Lihong Li, Michael~L. Littman, Thomas~J. Walsh, and Alexander~L. Strehl.
\newblock Knows what it knows: a framework for self-aware learning.
\newblock \emph{Machine Learning}, 82\penalty0 (3):\penalty0 399--443, March 2011.
\newblock ISSN 1573-0565.
\newblock \doi{10.1007/s10994-010-5225-4}.
\newblock URL \url{https://doi.org/10.1007/s10994-010-5225-4}.

\bibitem[Liu et~al.(2020)Liu, Swaminathan, Agarwal, and Brunskill]{liuProvablyGoodBatch2020c}
Yao Liu, Adith Swaminathan, Alekh Agarwal, and Emma Brunskill.
\newblock Provably {Good} {Batch} {Off}-{Policy} {Reinforcement} {Learning} {Without} {Great} {Exploration}.
\newblock In \emph{Advances in {Neural} {Information} {Processing} {Systems}}, volume~33, pages 1264--1274. Curran Associates, Inc., 2020.
\newblock URL \url{https://proceedings.neurips.cc/paper_files/paper/2020/hash/0dc23b6a0e4abc39904388dd3ffadcd1-Abstract.html}.

\bibitem[Madras et~al.(2018)Madras, Pitassi, and Zemel]{madras2018predict}
David Madras, Toni Pitassi, and Richard Zemel.
\newblock Predict responsibly: improving fairness and accuracy by learning to defer.
\newblock \emph{Advances in neural information processing systems}, 31, 2018.

\bibitem[Mozannar and Sontag(2020)]{mozannar2020consistent}
Hussein Mozannar and David Sontag.
\newblock Consistent estimators for learning to defer to an expert.
\newblock In \emph{International Conference on Machine Learning}, pages 7076--7087. PMLR, 2020.

\bibitem[Nadjahi et~al.(2019)Nadjahi, Laroche, and Combes]{nadjahiSafePolicyImprovement2019}
Kimia Nadjahi, Romain Laroche, and Rémi Tachet~des Combes.
\newblock Safe {Policy} {Improvement} with {Soft} {Baseline} {Bootstrapping}, July 2019.
\newblock URL \url{https://arxiv.org/abs/1907.05079v1}.

\bibitem[Scholl et~al.(2022)Scholl, Dietrich, Otte, and Udluft]{schollSafePolicyImprovement2022}
Philipp Scholl, Felix Dietrich, Clemens Otte, and Steffen Udluft.
\newblock Safe {Policy} {Improvement} {Approaches} and their {Limitations}, August 2022.
\newblock URL \url{http://arxiv.org/abs/2208.00724}.
\newblock arXiv:2208.00724 [cs] version: 1.

\bibitem[Shrestha et~al.(2020)Shrestha, Lee, Tadepalli, and Fern]{shresthaDeepAveragersOfflineReinforcement2020}
Aayam Shrestha, Stefan Lee, Prasad Tadepalli, and Alan Fern.
\newblock {DeepAveragers}: {Offline} {Reinforcement} {Learning} by {Solving} {Derived} {Non}-{Parametric} {MDPs}, October 2020.
\newblock URL \url{http://arxiv.org/abs/2010.08891}.
\newblock arXiv:2010.08891 [cs, stat].

\bibitem[Singh et~al.(2022)Singh, Kumar, Vuong, Chebotar, and Levine]{singhOfflineRLRealistic2022a}
Anikait Singh, Aviral Kumar, Quan Vuong, Yevgen Chebotar, and Sergey Levine.
\newblock Offline {RL} {With} {Realistic} {Datasets}: {Heteroskedasticity} and {Support} {Constraints}, November 2022.
\newblock URL \url{http://arxiv.org/abs/2211.01052}.
\newblock Issue: arXiv:2211.01052 arXiv:2211.01052 [cs, stat].

\bibitem[Straitouri et~al.(2021)Straitouri, Singla, Meresht, and Gomez-Rodriguez]{straitouriReinforcementLearningAlgorithmic2021}
Eleni Straitouri, Adish Singla, Vahid~Balazadeh Meresht, and Manuel Gomez-Rodriguez.
\newblock Reinforcement {Learning} {Under} {Algorithmic} {Triage}, September 2021.
\newblock URL \url{http://arxiv.org/abs/2109.11328}.
\newblock arXiv:2109.11328 [cs].

\bibitem[Sutton(1998)]{suttonMDPsSemiMDPsLearning}
Richard~S Sutton.
\newblock Between mdps and semi-mdps: Learning, planning, and representing knowledge at multiple temporal scales.
\newblock 1998.

\bibitem[Sutton and Barto(2018)]{sutton2018reinforcement}
Richard~S Sutton and Andrew~G Barto.
\newblock \emph{Reinforcement learning: An introduction}.
\newblock MIT press, 2018.

\bibitem[Thomas et~al.(2015{\natexlab{a}})Thomas, Theocharous, and Ghavamzadeh]{thomasHighConfidenceOffPolicyEvaluation2015}
Philip Thomas, Georgios Theocharous, and Mohammad Ghavamzadeh.
\newblock High-{Confidence} {Off}-{Policy} {Evaluation}.
\newblock \emph{Proceedings of the AAAI Conference on Artificial Intelligence}, 29\penalty0 (1), February 2015{\natexlab{a}}.
\newblock ISSN 2374-3468, 2159-5399.
\newblock \doi{10.1609/aaai.v29i1.9541}.
\newblock URL \url{https://ojs.aaai.org/index.php/AAAI/article/view/9541}.

\bibitem[Thomas et~al.(2015{\natexlab{b}})Thomas, Theocharous, and Ghavamzadeh]{thomasHighConfidencePolicy2015a}
Philip Thomas, Georgios Theocharous, and Mohammad Ghavamzadeh.
\newblock High {Confidence} {Policy} {Improvement}.
\newblock In \emph{Proceedings of the 32nd {International} {Conference} on {Machine} {Learning}}, pages 2380--2388. PMLR, June 2015{\natexlab{b}}.
\newblock URL \url{https://proceedings.mlr.press/v37/thomas15.html}.
\newblock ISSN: 1938-7228.

\bibitem[Wienhöft et~al.(2023{\natexlab{a}})Wienhöft, Suilen, Simão, Dubslaff, Baier, and Jansen]{wienhoftMoreLessSafe2023}
Patrick Wienhöft, Marnix Suilen, Thiago~D. Simão, Clemens Dubslaff, Christel Baier, and Nils Jansen.
\newblock More for {Less}: {Safe} {Policy} {Improvement} {With} {Stronger} {Performance} {Guarantees}, May 2023{\natexlab{a}}.
\newblock URL \url{http://arxiv.org/abs/2305.07958}.
\newblock arXiv:2305.07958 [cs].

\bibitem[Wienhöft et~al.(2023{\natexlab{b}})Wienhöft, Suilen, Simão, Dubslaff, Baier, and Jansen]{wienhoftMoreLessSafe2023a}
Patrick Wienhöft, Marnix Suilen, Thiago~D. Simão, Clemens Dubslaff, Christel Baier, and Nils Jansen.
\newblock More for less: safe policy improvement with stronger performance guarantees.
\newblock In \emph{Proceedings of the {Thirty}-{Second} {International} {Joint} {Conference} on {Artificial} {Intelligence}}, {IJCAI} '23, pages 4406--4415, {\textless}conf-loc{\textgreater}, {\textless}city{\textgreater}Macao{\textless}/city{\textgreater}, {\textless}country{\textgreater}P.R.China{\textless}/country{\textgreater}, {\textless}/conf-loc{\textgreater}, August 2023{\natexlab{b}}.
\newblock ISBN 978-1-956792-03-4.
\newblock \doi{10.24963/ijcai.2023/490}.
\newblock URL \url{https://doi.org/10.24963/ijcai.2023/490}.

\bibitem[Wu et~al.(2022)Wu, Wu, Qiu, Wang, and Long]{wu2022supported}
Jialong Wu, Haixu Wu, Zihan Qiu, Jianmin Wang, and Mingsheng Long.
\newblock Supported policy optimization for offline reinforcement learning.
\newblock \emph{Advances in Neural Information Processing Systems}, 35:\penalty0 31278--31291, 2022.

\bibitem[Yu et~al.(2020)Yu, Thomas, Yu, Ermon, Zou, Levine, Finn, and Ma]{yuMOPOModelbasedOffline2020}
Tianhe Yu, Garrett Thomas, Lantao Yu, Stefano Ermon, James Zou, Sergey Levine, Chelsea Finn, and Tengyu Ma.
\newblock {MOPO}: {Model}-based {Offline} {Policy} {Optimization}, November 2020.
\newblock URL \url{http://arxiv.org/abs/2005.13239}.
\newblock arXiv:2005.13239 [cs, stat].

\bibitem[Yu et~al.(2022)Yu, Kumar, Rafailov, Rajeswaran, Levine, and Finn]{yuCOMBOConservativeOffline2022}
Tianhe Yu, Aviral Kumar, Rafael Rafailov, Aravind Rajeswaran, Sergey Levine, and Chelsea Finn.
\newblock {COMBO}: {Conservative} {Offline} {Model}-{Based} {Policy} {Optimization}, January 2022.
\newblock URL \url{http://arxiv.org/abs/2102.08363}.
\newblock arXiv:2102.08363 [cs].

\bibitem[Zhang et~al.(2022)Zhang, Wang, Du, Chu, Arévalo, Kindle, Celi, and Doshi-Velez]{zhangInterpretableRLFramework2022}
Kristine Zhang, Henry Wang, Jianzhun Du, Brian Chu, Aldo~Robles Arévalo, Ryan Kindle, Leo~Anthony Celi, and Finale Doshi-Velez.
\newblock An interpretable {RL} framework for pre-deployment modeling in {ICU} hypotension management.
\newblock \emph{npj Digital Medicine}, 5\penalty0 (1):\penalty0 1--10, November 2022.
\newblock ISSN 2398-6352.
\newblock \doi{10.1038/s41746-022-00708-4}.
\newblock URL \url{https://www.nature.com/articles/s41746-022-00708-4}.
\newblock Publisher: Nature Publishing Group.

\end{thebibliography}

\onecolumn

\section{Proofs}
\label{sec:proofs}
\begin{lemma}[Performance Difference Lemma]
    \label{lem:perf_difference}
    Let $\pidp$ and $\pib$ be two policies. 
    Then, the difference in performance between the two policies is given by:
    \begin{gather}
        \rho(\pidp) - \rho(\pib)  = \frac{1}{1-\gamma} \E_{\eta^{\pidp}(s,a)} \brackets{ A^{\pib}(s, a) }
    \end{gather}
    where $A^{\pib}(s, a)$ is the advantage of taking action $a$ in state $s$ under policy $\pib$.
\end{lemma}

\begin{proof}
    See Lemma 1.16 of \cite{agarwalReinforcementLearningTheorya}.
\end{proof}
\begin{lemma}[McDiarmid's Inequality]
    \label{lem:mcdiarmid}
        
    Suppose a function $f: \gX_1 \times \gX_2 \times \ldots \times \gX_n \to \R$ satisfies the bounded differences property if:
    \begin{gather}
        \sup_{x_i',x_i \in \gX_i} f(x_1, \ldots , x_{i-1}, x_i', x_{i+1}, \ldots , x_n) - f(x_1, \ldots , x_{i-1}, x_i, x_{i+1}, \ldots , x_n) \le c_i.
    \end{gather}
    for all $i \in [n]$ and constants $c_1, \ldots, c_n$. 
    Let $X = (X_1, X_2, \ldots, X_n)$ be an $n$-tuple of independent random variables, where $X_i \in \gX_i$ for all $i \in [n]$.
    Then, for any $\epsilon > 0$,
    \begin{gather}
        \Pr[\abs{ f(X) - \E[f(X)] } \ge \epsilon] \le \exp \parens{ \frac{-2\epsilon^2}{\sum_{i=1}^n c_i^2} }
    \end{gather}
\end{lemma}

\begin{proof}
    See Theorem 2.3 of \citep{hajekECE543Statistical2019}
\end{proof}
\thmdpdiscrete*

\begin{proof}
    To prove the theorem, we first define a few useful quantities.
    Let $\Omega_s \subset \set{1,\cdots,N}$ be the set of trajectories that visit state $s$ and let $\Omega_{s,a} \subset \Omega_s$ be the set of trajectories that contain the $(s,a)$ pair.
    For $n \in \Omega_{s}$, let $t_n$ be the first time step that the state $s$ is visited in the trajectory $n$, i.e., $S^{n}_{t_n} = s$.
    For $n \in \Omega_{s,a}$, let $t_n^{a}$ be the first time step that the action $a$ is taken in the state $s$ in the trajectory $n$, i.e., $(S^{n}_{t_n^{a}}, A^{n}_{t_n^{a}}) = (s,a)$. Therefore, $t_n^{a} \geq t_n$.
    Let $G^n_{t}$ be the return of the trajectory $n$ starting from time step $t$, i.e., $G^n_{t} = \sum_{k=t}^{T_n} \gamma^{k-t} R^{n}_k$.
    We define the value function estimators $\hat{V}^{\pib}(s)$ and $\hat{Q}^{\pib}(s,a)$ as follows:
    \begin{gather}
        \hat{V}^{\pib}(s) = \frac{1}{\abs{\Omega_s}} \sum_{n \in \Omega_s} G^n_{t_n}, \quad
        \hat{Q}^{\pib}(s,a) = \frac{1}{\abs{\Omega_{s,a}}} \sum_{n \in \Omega_{s,a}} G^n_{t_n^{a}} \\
        \hat{A}^{\pib}(s,a) = \hat{Q}^{\pib}(s,a) - \hat{V}^{\pib}(s)
    \end{gather}
    Therefore, we have the following expectations:
    \begin{gather}
        \E[G^n_{t_n}] = V^{\pib}(s), \quad \text{and} \quad \E[G^n_{t_n^{a}}] = Q^{\pib}(s,a) \\
        \implies \E[\hat{V}^{\pib}(s)] = V^{\pib}(s), \quad \E[\hat{Q}^{\pib}(s,a)] = Q^{\pib}(s,a), \quad \E[\hat{A}^{\pib}(s,a)] = A^{\pib}(s,a)
    \end{gather}
    and for any $n \in \Omega_{s,a}$,
    \begin{align}
        G^n_{t_n} = G^n_{t_n:t_n^{a}} + \gamma^{t_n^{a}-t_n} G^n_{t_n^{a}}
    \end{align}
    where $G^n_{t_n:t_n^{a}}$ is the return of the trajectory $n$ from time step $t_n$ to $t_n^{a}$ (zero if $t_n = t_n^{a}$).

    Now, we can write the empirical advantage $\hat{A}^{\pib}(s,a)$ as:
    \begin{gather}
        \hat{A}^{\pib}(s,a) = \frac{1}{\abs{\Omega_{s,a}}} \sum_{n \in \Omega_{s,a}} G^n_{t_n^{a}} - \frac{1}{\abs{\Omega_s}} \sum_{n \in \Omega_s} G^n_{t_n} \\
        = \sum_{n \in \Omega_{s,a}} \frac{G^n_{t_n^{a}}}{\abs{\Omega_{s,a}}}  - \frac{\gamma^{t_n^{a}-t_n} G^n_{t_n^{a}}}{\abs{\Omega_s}}
        - \sum_{n \in \Omega_s} \frac{G^n_{t_n:t_n^{a}}}{\abs{\Omega_s}}
        - \sum_{n \in \Omega_s \setminus \Omega_{s,a}} \frac{G^n_{t_n}}{\abs{\Omega_s}} \\
        = S_1(\set{G^n_{t_n^{a}}}_{n \in \Omega_{s,a}}) + S_2(\set{G^n_{t_n:t_n^{a}}}_{n \in \Omega_s}) + S_3(\set{G^n_{t_n}}_{n \in \Omega_s})
    \end{gather}
    where all the random variables $\gG := \set{G^n_{t_n^{a}}}_{n \in \Omega_{s,a}} \cup \set{G^n_{t_n:t_n^{a}}}_{n \in \Omega_s} \cup \set{G^n_{t_n}}_{n \in \Omega_s}$ are independent and bounded in $[0, \Vmax]$.

    We can now apply McDiarmid's inequality to the function $f(\set{G^n_{t_n^{a}}}_{n \in \Omega_{s,a}}, \set{G^n_{t_n:t_n^{a}}}_{n \in \Omega_s}, \set{G^n_{t_n}}_{n \in \Omega_s}) = S_1 + S_2 + S_3$.
    Therefore, we have:
    \begin{gather}
        \Pr\brackets{f(\gG) - \E[f(\gG)] \geq \epsilon} \leq \exp\parens{-\frac{2 \epsilon^2 \Nmin}{2 \Vmax^2}} \\
        \implies \Pr\brackets{A^{\pib}(s,a) \geq -\epsilon - \hat{A}^{\pib}(s,a)} \leq \exp\parens{-\frac{2 \epsilon^2 \Nmin}{2 \Vmax^2}} \\
        \implies \Pr\brackets{A^{\pib}(s,a) \geq -\epsilon} \leq \exp\parens{-\frac{2 \epsilon^2 \Nmin}{2 \Vmax^2}}
    \end{gather}
    where the last inequality follows from the fact that $\hat{A}^{\pib}(s,a) \geq 0$ and $\Nmin$ is the minimum number of times the $(s,a)$ pair is observed in the dataset.
    
    Apply the union bound over all valid state-action pairs $C(\Nmin) = \sum_{s \in \gS} \sum_{a \in \gA} \indicator\brackets{n(s,a) \geq \Nmin}$ to get the desired result that with probability at least $1-\delta$:
    \begin{gather}
        A^{\pib}(s,a) \geq -\Vmax \sqrt{\frac{1}{\Nmin} \log \frac{C(\Nmin)}{\delta}}
    \end{gather}
    Note that we did not need to apply the union bound over the rest of the state-action pairs because the advantage is zero as $\pidp = \pib$ in those states.
     
    Finally, we can use the performance difference lemma to get the desired result:
    \begin{gather}
        \rho(\pidp) - \rho(\pib) = \frac{1}{1-\gamma} \E[A^{\pib}(s,a)] \geq - \frac{\Vmax}{1-\gamma} \sqrt{\frac{1}{2 \Nmin} \log \frac{C(\Nmin)}{\delta}}
    \end{gather}
    
\end{proof}
\thmdpcontinuous*

\begin{proof}
    The proof for the discrete case is identical except that we pay a penalty of $\epsilon_r$ everytime we use $(s,a)$'s neighbor's action-value ($Q(s',a')$ for $(s',a') \in B_r(s,a)$) to estimate $Q(s,a)$. We create a covering $\gC$ of the dense region of the dataset. We assume $M(r,\Nmin) \geq \abs{\gC}$ 
    \begin{enumerate}
        \item For any point $(s,a)$ with at least $\Nmin$ neighbors in $B_r(s,a)$, the advantage $A{\pib}(s,a)$ is bounded from below because it is a monte-carlo average (like in discrete case analysis). However, using the returns of the neighbors incurs an $\epsilon_r$ error.
        \item For any point $(s',a') \in \gC$, there is at least one neighbor with at least $\Nmin$ neighbors. We incur another $\epsilon_r$ error. Then we apply union bound over $M(r, \Nmin)$ points.
        \item Finally, we incur another $\epsilon_r$ error for going from points in $\gC$ to all the points where $A^{\pib}(s,a) \geq 0$ (i.e. we don't defer).
    \end{enumerate}
\end{proof}

\section{Bounds in Prior Work}
\label{sec:prior_bounds}

% % Mention that all of them have a SA dependence and we don't.
Here we reproduce the bounds in prior work for reference.
The bound by \citep{larocheSafePolicyImprovement2019} (Theorem 2) is, with probability $1-\delta$,
 \begin{gather}
     \rho(\hat{\pi}) - \rho(\pib) \geq - \frac{4 \Vmax}{1 - \gamma} \sqrt{ \frac{2}{\Nmin} \log \frac{2 \abs{\gS} \abs{\gA} 2^{\abs{\gS}}}{\delta}} - \rho(\pi, \hat{M}) + \rho(\pib, \hat{M}) 
 \end{gather}
 This bound also depends on the $\Nmin$ parameter and $(1-\gamma)$ in the same way as our bound. 
 However, the bound differs from ours because it has a dependence on the number of states and actions. This is much larger than the $C(\Nmin)$ term in our bound, which only scales with the number of $(s,a)$ pairs that are observed at least $\Nmin$ times.

 The bound by \citep{liuProvablyGoodBatch2020c} (Corollary 2) is, with probability $1-\delta$,
 \begin{gather}
     \rho(\hat{\pi}) - \rho(\pib) \geq - \tilde{\gO} \parens{ \frac{\Vmax}{b (1-\gamma)^3} \frac{\abs{\gS} \abs{\gA}}{n} + \frac{\Vmax}{b (1-\gamma)^3} \sqrt{\frac{\abs{\gS} \abs{\gA}}{n}} + \frac{\gamma^K \Vmax}{(1-\gamma)^2}}
 \end{gather}
 This bound differs from ours because it has a dependence on $b := \min_{s,a} \eta^{\pib}(s,a)$, and the bound can be even more loose when $\hat{\eta}^{\pib}$ is used to estimate $\eta^{\pib}$. Furthermore, the bound also has a dependence on the number of states and actions, which can be large in practice.
 In contrast, our bound is directly in terms of the $\Nmin$ parameter, and does not have to assume any bounds on the data distribution. This allows for a tighter bound in practice even when the excluded $(s,a)$ pairs are the same.

 The bound by \citep{kimModelbasedOfflineReinforcement2023b} (Theorem 2) is, with probability $1-\delta$,
 \begin{gather}
     \rho(\hat{\pi}) - \rho(\pib) \geq -\frac{\gamma \Vmax}{(1-\gamma)^2} \E_{(s,a) \sim \eta^{\pib}_{\hat{P}}} \brackets{ \min\parens{1, \sqrt{ \frac{2}{n(s,a)} \log \frac{\abs{\gS} \abs{\gA} }{\delta} }} }
 \end{gather}
 where $\eta^{\pib}_{\hat{P}}(s,a)$ is the state-action visitation distribution under the behavior policy and the MLE transition model.
 This bound is similar to ours in that it depends on the count $n(s,a)$, but it has a dependence on the number of states and actions. Furthermore, the dependence on $n(s,a)$ is only useful when the count is large, and the bound can be loose when there are many $(s,a)$ pairs with low counts. We avoid this by directly controlling the set of $(s,a)$ pairs that are included in the policy set.

\section{Algorithms}
\subsection{Algorithm: DPRL-Discrete}
\begin{algorithm}
\caption{DPRL-D for Discrete States}
\label{alg:dp_discrete}
\begin{algorithmic}[1]
\STATE \textbf{Input:} Dataset, $\gD = \set{ (S^n_t, A^n_t, R^n_t)}$
\STATE Compute $\hat{Q}^{\pi_b}, \hat{V}^{\pi_b}$ using the dataset $\gD$
\STATE Compute $\gSdp, \gAdp = \set{\gAdp_s: s \in \gSdp}$ using Eq~\ref{eq:advantageous_actions_discrete}
\STATE $\Pdp, \Rdp, \gammadp \leftarrow \text{Make SMDP Parameters}(\gD, \gSdp)$
\STATE \textbf{Policy Iteration:}
\STATE \textbf{Initialize:} $i \leftarrow 1$ and $V^{(1)} \leftarrow \hat{V}^{\pi_b}$
\STATE $\pi^{(1)}(s) \leftarrow \arg\max_{a \in \gAdp_s} \hat{Q}^{\pi_b}(s, a) \quad \forall s \in \gSdp$
\REPEAT
\STATE $i \leftarrow i + 1$
\STATE $V^{(i)} \leftarrow \text{PolicyEval} \left[ \Pdp, \Rdp, \pi^{(i)} \right]$
% \STATE $\pi^{(i)}(s) \leftarrow \arg\max_{a \in \gAdp_s} \Rdp(s, a) + \E_{\Pdp(s' | s, a)} \gammadp(s, a, s') V^{(i)}(s') \quad \forall s \in \gSdp$
\STATE Update $\pi^{(i)}(s)$ using Eq~\ref{eq:policy_improvement}
\UNTIL{$\infnorm{V^{(i)} - V^{(i-1)}} \leq \varepsilon$}
\STATE \textbf{Output:} $\pi^{(i)}$
\end{algorithmic}
\end{algorithm}
\subsection{Algorithm to Construct SMDP Parameters for Discrete States}
\label{sec:make_smdp}
\begin{algorithm}[H]
    \caption{Make SMDP Parameters}
    \label{alg:make_smdp}
    \begin{algorithmic}[1]
    \STATE \textbf{Input:} Dataset, $\gD = \set{ (S^n_t, A^n_t, R^n_t) : n = 1, \ldots, N, t = 1, \ldots, T_n }$
    \STATE \textbf{Input:} Decision points, $\gSdp \subset \gS$
    
    \STATE $\gSdp \leftarrow \text{set of all possible states}$
    
    \FOR{$n$ in $[1, \cdots N]$}
    \STATE $\tau^n \leftarrow \text{dict()}$
    \FOR{$t$ in $[1, T_n]$}
    \IF{$S^n_t \notin \tau^n$ and $S^n_t \in \gSdp$}
    \STATE $\tau^n[S^n_t] \leftarrow t$
    \ENDIF
    
    \STATE $\tau^n_{\text{sorted}} \leftarrow \text{sort}(\tau^n)$
    
    \FOR{$t, t'$ in zip($\tau^n_{\text{sorted}}[:-1], \tau^n_{\text{sorted}}[1:]$)}
    \STATE $Y(S^n_t, A^n_t, S^n_{t'}) \leftarrow Y(S^n_t, A^n_t, S^n_{t'}) + \gamma^{t' - t}$
    \STATE $G(S^n_t, A^n_t, S^n_{t'}) \leftarrow G(S^n_t, A^n_t, S^n_{t'}) + \sum_{k=t}^{t'} \gamma^{k - t} R^n_k$
    \STATE $\ndp(S^n_t, A^n_t, S^n_{t'}) \leftarrow \ndp(S^n_t, A^n_t, S^n_{t'}) + 1$
    \ENDFOR
    \ENDFOR
    \ENDFOR
    
    \STATE $\gammadp(s, a, s') \leftarrow Y(s, a, s') / \ndp(s, a, s')$
    \STATE $\Pdp(s' | s, a) \leftarrow {\ndp(s, a, s')} / {\sum_{s' \in \gSdp} \ndp(s, a, s')}$
    \STATE $\rdp(s, a, s') \leftarrow G(s, a, s') / \ndp(s, a, s')$
    \STATE $\Rdp(s,a) \leftarrow \sum_{s' \in \gSdp} \rdp(s, a, s') \Pdp(s, a, s')$
    
    \STATE \textbf{Output:} $\Pdp, \Rdp, \gammadp$
    \end{algorithmic}
\end{algorithm}

\subsection{Algorithm: DPRL-Continuous}

\begin{algorithm}[H]
    \caption{DPRL-C for Continuous States}
    \label{alg:dp_continuous}
    \begin{algorithmic}[1]
    \STATE \textbf{Input:} State $s$
    \STATE \textbf{Input:} BallTree$_{sa}$($\cdot, \gD, r$), BallTree$_s$($\cdot, \gD, r$)
    \STATE $\gN(s) \leftarrow$ BallTree$_s$($s, \gD, r$)
    \STATE $\gN(s,a) \leftarrow$ BallTree$_{sa}$($(s,a), \gD, r$) for all $a \in \gA$
    \IF{$\abs{\gN(s)} \leq \Nmin$}
        \STATE \textbf{Output:} DEFER
    \ELSE
        \STATE Compute $\hat{V}^{\pi_b}(s)$ using $\gN(s)$
        \STATE Compute $\hat{Q}^{\pi_b}(s, a)$ using $\gN(s, a)$
        \STATE Compute $\gAdp_s$ using Eq~\ref{eq:advantageous_actions_continuous}
        \IF{$\gAdp_s = \emptyset$}
            \STATE \textbf{Output:} DEFER
        \ELSE
            \STATE \textbf{Output:} $\argmax_{a \in \gAdp_s} \hat{Q}^{\pi_b}(s, a)$
        \ENDIF
    \ENDIF
    \end{algorithmic}
\end{algorithm}

\section{Details about the Synthethic MDPs}
\label{sec:appdx_mdps}

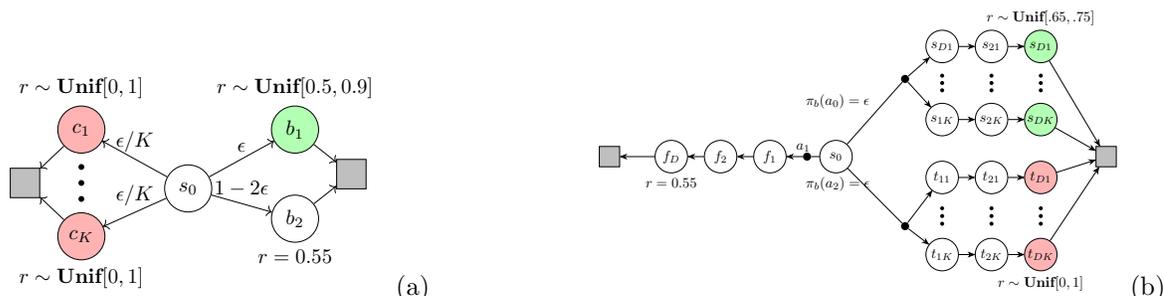
\begin{figure}[h]
    \centering
    \begin{minipage}[b]{0.45\linewidth}
        \centering
        \resizebox{.65\textwidth}{!}{\begin{tikzpicture}[
      % vertical and horizontal node distances
      node distance=0.7cm and 0.4cm,
      state/.style={draw, circle, minimum size=8mm, inner sep=0pt}, % define state nodes and apply the circular style
      action/.style={fill, circle, minimum size=2pt, inner sep=2pt}, % define action nodes (small filled circles)
      square/.style={draw, rectangle, minimum size=5mm, inner sep=0pt, fill=gray!50}, % define square node style
      dots/.style={fill, circle, minimum size=1pt, inner sep=1pt}, % smaller dots
      every edge/.style={draw, -{Stealth[round]}, thick},
    ]

    % Center node S_0
    \node[state] (s0) at (0, 0) {$s_0$};

    % Nodes b1 and b2 to the right of s0
    \node[state, fill=green!30] (b1) [right=1cm of s0, yshift=1cm] {$b_1$};
    \node[state] (b2) [below=of b1] {$b_2$};

    \node[above] at (b1.north) {$r \sim \textbf{Unif}[0.5, 0.9]$};
    \node[below] at (b2.south) {$r = 0.55$};

    % Nodes c1 and cK to the left of s0
    \node[state, fill=red!30] (c1) [left=1cm of s0, yshift=1cm] {$c_1$};
    \node[state, fill=red!30] (cK) [below=1cm of c1] {$c_K$};

    \node[above] at (c1.north) {$r \sim \textbf{Unif}[0, 1]$};
    \node[below] at (cK.south) {$r \sim \textbf{Unif}[0, 1]$};

    % Square node aligned vertically between b1 and b2, and c1 and cK
    \node[square] (sq_b) at ($(b1)!0.5!(b2)$) [right=0.7cm] {}; % square aligned vertically between b1 and b2
    \node[square] (sq_c) at ($(c1)!0.5!(cK)$) [left=0.7cm] {};  % square aligned vertically between c1 and cK

    % Arrows between s0 and b1, b2
    \draw[->] (s0) -- (b1) node[midway, above] {$\epsilon$};
    \draw[->] (s0) -- (b2) node[midway, above] {$1-2\epsilon$};

    % Arrows between s0 and c1, cK
    \draw[->] (s0) -- (c1) node[midway, above] {$\epsilon/K$};
    \draw[->] (s0) -- (cK) node[midway, above] {$\epsilon/K$};

    % Arrows from b1 and b2 to square node on the right
    \draw[->] (b1) -- (sq_b);
    \draw[->] (b2) -- (sq_b);

    % Arrows from c1 and cK to square node on the left
    \draw[->] (c1) -- (sq_c);
    \draw[->] (cK) -- (sq_c);

    % Manually add 3 smaller, closer non-action dots between c1 and cK
    \node[dots] at ($(c1)!0.35!(cK)$) {}; % 35% of the way from c1 to cK
    \node[dots] at ($(c1)!0.50!(cK)$) {}; % 50% of the way from c1 to cK
    \node[dots] at ($(c1)!0.65!(cK)$) {}; % 65% of the way from c1 to cK

\end{tikzpicture}
%     \caption{Toy MDP for CQL. The optimal action leads to state $b_1$, and the suboptimal action (frequent under behavior) leads to state $b_2$. There are also risky states $\set{c_1, \dots, c_K}$. TODO: LABEL THE ACTIONS IN THE MDP.}
%     \label{fig:toy_mdp_cql}
% \end{figure}
}
        (a)
    \end{minipage}
    \hspace{0.05\linewidth} % Horizontal space between the two figures
    \begin{minipage}[b]{0.45\linewidth}
        \centering
        \resizebox{.9\textwidth}{!}{\begin{tikzpicture}[
      % Define styles for different node types
      node distance=0.7cm and 0.4cm, % Vertical and horizontal spacing between nodes
      state/.style={draw, circle, minimum size=8mm, inner sep=0pt}, % Circular state nodes
      action/.style={fill, circle, minimum size=2pt, inner sep=2pt}, % Small filled action nodes
      square/.style={draw, rectangle, minimum size=5mm, inner sep=0pt, fill=gray!50}, % Square node with gray fill
      dots/.style={fill, circle, minimum size=1pt, inner sep=1pt}, % Small dots for intermediary connections
      every edge/.style={draw, -{Stealth[round]}, thick}, % Arrow style
      >=Stealth
    ]

    % ------------------------------
    % Initial Node
    % ------------------------------
    \node[state] (s0) {$s_0$}; % Starting state of the MDP

    % ------------------------------
    % Action Nodes
    % ------------------------------
    \node[action] (a0) [above=1cm of s0, xshift=1.7cm, yshift=.4cm] {}; % Action leading to upper branches
    \node[action] (a1) [left=.2cm of s0] {}; % Action leading to middle branch
    \node[action] (a2) [below=1cm of s0,  xshift=1.7cm, yshift=-.2cm] {}; % Action leading to lower branches

    % ------------------------------
    % Connections from Initial Node to Actions
    % ------------------------------
    \draw[] (s0) -- (a0) node[midway, above left] {$\pi_b(a_0)=\epsilon$}; % Edge from s0 to a0 with label a0
    \draw[] (s0) -- (a1) node[midway, above left] {$a_1$}; % Edge from s0 to a1 with label a1
    \draw[] (s0) -- (a2) node[midway, above left] {$\pi_b(a_2)=\epsilon$}; % Edge from s0 to a2 with label a2

    % ------------------------------
    % Final Node
    % ------------------------------
    \node[square] (z) [right=7cm of a1] {}; % Final square node where all chains converge
    \node[square] (z1) [left=4.5cm of a1] {};

    % ------------------------------
    % Upper Branches (Action a0)
    % ------------------------------
    % Chain 1: s1D -> s2D -> sD1 -> z
    \node[state] (s1D) [right=of a0, yshift=0.8cm] {$s_{D1}$}; % First state in Chain 1 (indexed by D)
    \draw[->] (a0) -- (s1D); % Edge from a0 to s1D

    \node[state] (s2D) [right=of s1D] {$s_{21}$}; % Second state in Chain 1 (indexed by D)
    \draw[->] (s1D) -- (s2D); % Edge from s1D to s2D

    \node[state, fill=green!30] (sD1) [right=of s2D] {$s_{D1}$}; % Third state in Chain 1 (indexed by K)
    \draw[->] (s2D) -- (sD1); % Edge from s2D to sD1

    \draw[->] (sD1) -- (z); % Edge from sD1 to final node z

    % Chain 2: s1K -> s2K -> sDK -> z
    \node[state] (s1K) [right=of a0, yshift=-1cm] {$s_{1K}$}; % First state in Chain 2 (indexed by K)
    \draw[->] (a0) -- (s1K); % Edge from a0 to s1K

    \node[state] (s2K) [right=of s1K] {$s_{2K}$}; % Second state in Chain 2 (indexed by K)
    \draw[->] (s1K) -- (s2K); % Edge from s1K to s2K

    \node[state, fill=green!30] (sDK) [right=of s2K] {$s_{DK}$}; % Third state in Chain 2 (indexed by K)
    \draw[->] (s2K) -- (sDK); % Edge from s2K to sDK

    \draw[->] (sDK) -- (z); % Edge from sDK to final node z
    \node[above] at (sD1.north) {$r \sim \textbf{Unif}[.65, .75]$};

    % ------------------------------
    % Intermediary Dots Between Upper Chains
    % ------------------------------
    % Dots between s1D and s1K
    \node[dots] at ($(s1D)!0.4!(s1K)$) {}; % First dot
    \node[dots] at ($(s1D)!0.5!(s1K)$) {}; % Second dot
    \node[dots] at ($(s1D)!0.6!(s1K)$) {}; % Third dot

    % Dots between s2D and s2K
    \node[dots] at ($(s2D)!0.4!(s2K)$) {}; % First dot
    \node[dots] at ($(s2D)!0.5!(s2K)$) {}; % Second dot
    \node[dots] at ($(s2D)!0.6!(s2K)$) {}; % Third dot

    % Dots between sD1 and sDK
    \node[dots] at ($(sD1)!0.4!(sDK)$) {}; % First dot
    \node[dots] at ($(sD1)!0.5!(sDK)$) {}; % Second dot
    \node[dots] at ($(sD1)!0.6!(sDK)$) {}; % Third dot

    % ------------------------------
    % Middle Branch (Action a1)
    % ------------------------------
    \node[state] (f1) [left=of a1] {$f_1$}; % First state in middle chain
    \draw[->] (a1) -- (f1); % Edge from a1 to f1

    \node[state] (f2) [left=of f1] {$f_2$}; % Second state in middle chain
    \draw[->] (f1) -- (f2); % Edge from f1 to f2

    \node[state] (fD) [left=of f2] {$f_{D}$}; % Final state in middle chain (indexed by D)
    \draw[->] (f2) -- (fD); % Edge from f2 to fD

    \draw[->] (fD) -- (z1); % Edge from fD to final node z
    \node[below] at (fD.south) {$r=0.55$};

    % ------------------------------
    % Lower Branches (Action a2)
    % ------------------------------
    % Chain 1: t1D -> t2D -> tD1 -> z
    \node[state] (t1D) [right=of a2, yshift=1.2cm] 
    {$t_{11}$}; % First state in Chain 1 (indexed by D)

    \draw[->] (a2) -- (t1D); % Edge from a2 to t1D

    \node[state] (t2D) [right=of t1D] {$t_{21}$}; % Second state in Chain 1 (indexed by D)
    \draw[->] (t1D) -- (t2D); % Edge from t1D to t2D

    \node[state, fill=red!30] (tD1) [right=of t2D] {$t_{D1}$}; % Third state in Chain 1 (indexed by K)

    \draw[->] (t2D) -- (tD1); % Edge from t2D to tD1

    \draw[->] (tD1) -- (z); % Edge from tD1 to final node z

    % Chain 2: t1K -> t2K -> tDK -> z
    \node[state] (t1K) [right=of a2, yshift=-0.7cm] {$t_{1K}$}; % First state in Chain 2 (indexed by K)
    \draw[->] (a2) -- (t1K); % Edge from a2 to t1K

    \node[state] (t2K) [right=of t1K] {$t_{2K}$}; % Second state in Chain 2 (indexed by K)
    \draw[->] (t1K) -- (t2K); % Edge from t1K to t2K

    \node[state, fill=red!30] (tDK) [right=of t2K] {$t_{DK}$}; % Third state in Chain 2 (indexed by K)
    \draw[->] (t2K) -- (tDK); % Edge from t2K to tDK
    \node[below] at (tDK.south) {$r \sim \textbf{Unif}[0, 1]$};
    \draw[->] (tDK) -- (z); % Edge from tDK to final node z

    % ------------------------------
    % Intermediary Dots Between Lower Chains
    % ------------------------------
    % Dots between t1D and t1K
    \node[dots] at ($(t1D)!0.4!(t1K)$) {}; % First dot
    \node[dots] at ($(t1D)!0.5!(t1K)$) {}; % Second dot
    \node[dots] at ($(t1D)!0.6!(t1K)$) {}; % Third dot

    % Dots between t2D and t2K
    \node[dots] at ($(t2D)!0.4!(t2K)$) {}; % First dot
    \node[dots] at ($(t2D)!0.5!(t2K)$) {}; % Second dot
    \node[dots] at ($(t2D)!0.6!(t2K)$) {}; % Third dot

    % Dots between tD1 and tDK
    \node[dots] at ($(tD1)!0.4!(tDK)$) {}; % First dot
    \node[dots] at ($(tD1)!0.5!(tDK)$) {}; % Second dot
    \node[dots] at ($(tD1)!0.6!(tDK)$) {}; % Third dot

    \end{tikzpicture}
%     \caption{Toy MDP for PQI. The optimal action ($a_0$) leads to the upper branch with two parallel chains indexed by depth ($D$) and chain number ($K$), $a_1$ leads to the middle branch, and $a_2$ leads to the lower branch with two parallel chains indexed by depth ($D$) and chain number ($K$). All branches converge to the final state $z$. Intermediary dots indicate additional potential connections between parallel chains.}
%     \label{fig:toy_mdp_pqi}
% \end{figure}
}
        (b) 
    \end{minipage}
    
    \caption{(a) Toy MDP for CQL: The optimal action leads to state $b_1$, and the suboptimal action (frequent under behavior) leads to state $b_2$. There are also risky states $\{c_1, \dots, c_K\}$. (b) Toy MDP for PQI: The optimal action ($a_0$) leads to the upper branch with two parallel chains indexed by depth ($D$) and chain number ($K$). $a_1$ leads to the middle branch, and $a_2$ leads to the lower branch with two parallel chains indexed by depth ($D$) and chain number ($K$). All branches converge to the final state $z$. Intermediary dots indicate additional potential connections between parallel chains.}
    \label{fig:combined}
\end{figure}

\paragraph{Experimental details for the bounds plot.}
To create the bounds plot, we simulated the forests MDP for 50 trials. For each trial, we generated a new dataset, and for each dataset, we computed the bounds for three methods: DP, SPIBB, and PQI. We varied parameters $\Nmin$ in $\set{1, 10, 100}$ and the number of states to observe their effect on the bounds (we varied the states by varying the number of chains in the `forests" in $\set{10, 20, 30, 50}$. We set the density threshold for PQI to $0.02$ for all the simulations.

\section{Gridworld Experimental Details}

The GridWorld has 100 states and 4 actions. The agent must start at the bottom left cell and reach the top right cell. The dynamics are stochastic with going to the intended cell with a probability of $0.9$ and a simulating a random action otherwise. The rewards are stochastic, and described in Figure~\ref{fig:gw_results}(A).
We sampled 500 random datasets from the environment to evaluate the reliability of the algorithms.
For DPRL, SPIBB and PQI, we tested the $\Nmin$ parameter in $\set{1,2,5,10,20,30}$ and found $\Nmin=20$ to be good for all the algorithms. For CQL, we varied $\alpha$ in $\set{0.01,0.05,0.1}$.

\section{Atari Experimental Details}

We use Atari datasets and environments for our experiments, specifically focusing on five environments: \textbf{Qubert}, \textbf{Pong}, \textbf{Freeway}, \textbf{Booling}, and \textbf{Amidar}. These environments allow for a diverse range of behaviors and complexities to evaluate our methods effectively.

\paragraph{Policy Training}
\begin{itemize}
    \item \textbf{Optimal Policy}: We trained an optimal policy for each environment by running a \textbf{Deep Q-Network (DQN)} for \textbf{10 million steps}. The \textbf{Stable Baselines} implementation of DQN was used for this training.
    \item \textbf{Medium Policy}: In addition to the optimal policy, we defined a "medium" policy by using the state of the DQN after \textbf{1 million steps} of training, which represents a less performant but reasonable policy.
\end{itemize}

\paragraph{Trajectory Simulation}
We simulated trajectories for each environment with a structured deviation from the optimal policy:
\begin{itemize}
    \item \textbf{Action Selection Probability}: For each action in the trajectory, the probability of taking the optimal action was set to \textbf{0.5}, while the probability of taking an action based on the medium policy was also set to \textbf{0.5}.
    \item \textbf{Data Generation}: We simulated \textbf{100,000 samples} for each environment, which correspond to a different number of episodes per environment depending on their characteristics.
\end{itemize}

This setup creates a structured deviation from the optimal policy, where the action taken can either be optimal or highly suboptimal, making the dataset suitable for evaluating off-policy methods.

\paragraph{Representation Learning}
To ensure consistency across all methods, we used a \textbf{16-dimensional representation} for state-action pairs, which were learned using the collected datasets. This step standardizes the inputs for the comparison of various algorithms.

\begin{itemize}
    \item \textbf{State Representation}: We used the mean of all action representations corresponding to each state to derive the state representations.
    \item \textbf{Data Storage}: We utilized a \textbf{BallTree} data structure to store the data, enabling efficient nearest neighbor searches during evaluation.
\end{itemize}

\paragraph{Implementation Details}
We compared several methods using the learned state-action representations, with the following specific implementations:

\begin{itemize}
    \item \textbf{PQI and SPIBB}: We used \textbf{ExtraTrees} (with \textbf{100 trees}) to estimate Q-values. ExtraTrees offers a highly flexible function approximation for this purpose.
    
    \item \textbf{CQL}: Since ExtraTrees is unsuitable for CQL (due to the nature of its loss function), we employed a \textbf{DQN architecture} with a single hidden layer of \textbf{16 nodes} for Q-value estimation.
    
    \item \textbf{PQI Density Estimation}: For PQI, we estimated the joint density of state-action pairs by counting the number of neighbors in the vicinity of each state-action pair using the \textbf{BallTree} data structure. This count was then divided by the dataset size to obtain the density estimate. Although this is a heuristic, it provides a feasible approach for density estimation in continuous state-action spaces.
    \begin{itemize}
        \item \textbf{Note}: The original PQI paper employed a \textbf{variational auto-encoder} for density estimation, which was applied only to states, not state-action pairs.
    \end{itemize}
    
    \item \textbf{SPIBB Behavior Policy}: We estimated the behavior policy using \textbf{random forests}, and the density of state-action pairs was computed in the same way as for PQI using the \textbf{BallTree} data structure.
\end{itemize}

\paragraph{Hyperparameters}
We conducted extensive experiments with the following hyperparameter settings (see Table \ref{tab:best_hyperparameters_atari} for the best hyperparameters):

\begin{itemize}
    \item \textbf{DPRL}: \( \Nmin \in \{5, 10, 30\} \)
    \item \textbf{SPIBB}: \( \Nmin \in \{5, 10, 30\} \)
    \item \textbf{PQI}: Density threshold (\( b \)) \( \in \{1e-4, 1e-3, 1e-2\} \)
    \item \textbf{Radius for neighbor search}: \( 1e-3 \) (for all methods)
    \item \textbf{CQL}: Alpha parameter \( \alpha \in \{1e-4, 1e-2, 1e-1\} \)
\end{itemize}

\begin{table}
\centering
\caption{Best hyperparameter combinations for each environment and algorithm (Atari).}
\label{tab:best_hyperparameters_atari}
\begin{tabular}{lll}
\toprule
 &  & hyperparameters \\
 \\
\midrule
\multirow[t]{4}{*}{AmidarNoFrameskip-v4} & CQL & $\alpha=0.0001$ \\
 & DP & $N_{\wedge}=5.0$, $r=0.0001$ \\
 & PQI & $b=0.001$, $r=0.001$ \\
 & SPIBB & $N_{\wedge}=5.0$, $r=0.001$ \\
\cline{1-3}
\multirow[t]{4}{*}{BowlingNoFrameskip-v4} & CQL & $\alpha=0.01$ \\
 & DP & $N_{\wedge}=5.0$, $r=0.0001$ \\
 & PQI & $b=0.01$, $r=0.001$ \\
 & SPIBB & $N_{\wedge}=10.0$, $r=0.001$ \\
\cline{1-3}
\multirow[t]{4}{*}{FreewayNoFrameskip-v4} & CQL & $\alpha=0.01$ \\
 & DP & $N_{\wedge}=5.0$, $r=0.0001$ \\
 & PQI & $b=0.0001$, $r=0.001$ \\
 & SPIBB & $N_{\wedge}=30.0$, $r=0.001$ \\
\cline{1-3}
\multirow[t]{4}{*}{PongNoFrameskip-v4} & CQL & $\alpha=0.1$ \\
 & DP & $N_{\wedge}=5.0$, $r=0.0001$ \\
 & PQI & $b=0.001$, $r=0.001$ \\
 & SPIBB & $N_{\wedge}=10.0$, $r=0.001$ \\
\cline{1-3}
\multirow[t]{4}{*}{QbertNoFrameskip-v4} & CQL & $\alpha=0.1$ \\
 & DP & $N_{\wedge}=5.0$, $r=0.0001$ \\
 & PQI & $b=0.0001$, $r=0.001$ \\
 & SPIBB & $N_{\wedge}=30.0$, $r=0.001$ \\
\cline{1-3}
\bottomrule
\end{tabular}
\end{table}

% Table 2: Architecture used by each DQN and Representation Learning Network (Atari)
\begin{table}[h]
\centering
\caption{Architecture used by each DQN and Representation Learning Network (Atari)}
\label{tab:network_architecture}
\begin{tabular}{p{4cm}c p{5cm}}
\toprule
\textbf{Layer}                                 & \textbf{Number of outputs} & \textbf{Other details}                                      \\ \midrule
Input frame size                               & (4x84x84)                  & –                                                           \\ 
Downscale convolution 1                        & 12800                      & kernel 8x8, depth 32, stride 4x4                             \\ 
Downscale convolution 2                        & 5184                       & kernel 4x4, depth 32, stride 2x2                             \\ 
Downscale convolution 3                        & 3136                       & kernel 3x3, depth 32, stride 1x1                             \\ 
Layers in feedforward network of DQN           & [512, 64, 16, \(|A|\)]      & –                                                           \\ 
Layers in feedforward network of representation learning model & [512, 16 * \(|A|\)]      & layer norm, \texttt{tanh} activation after the last layer    \\ \bottomrule
\end{tabular}
\end{table}

% Table 3: All Hyperparameters for DQN
\begin{table}[h]
\centering
\caption{All Hyperparameters for DQN used for training the expert policy}
\label{tab:dqn_hyperparameters}
\begin{tabular}{lp{6cm}} % Adjust column width here as needed
\toprule
\textbf{Hyperparameter}               & \textbf{Value}                         \\ \midrule
Network optimizer                      & Adam               \\ 
Learning rate                          & 0.0001                                \\ 
Adam $\epsilon$                        & 0.000015                              \\ 
Discount $\gamma$                      & 0.99                                  \\ 
Mini-batch size                        & 128                                   \\ 
Target network update frequency        & 10k training iterations               \\ 
Evaluation $\epsilon$                  & 0.001                                 \\ \bottomrule
\end{tabular}
\end{table}

\subsection{Additional Plots - Atari}

\begin{figure}[H]
    \centering
    \includegraphics[width=\textwidth]{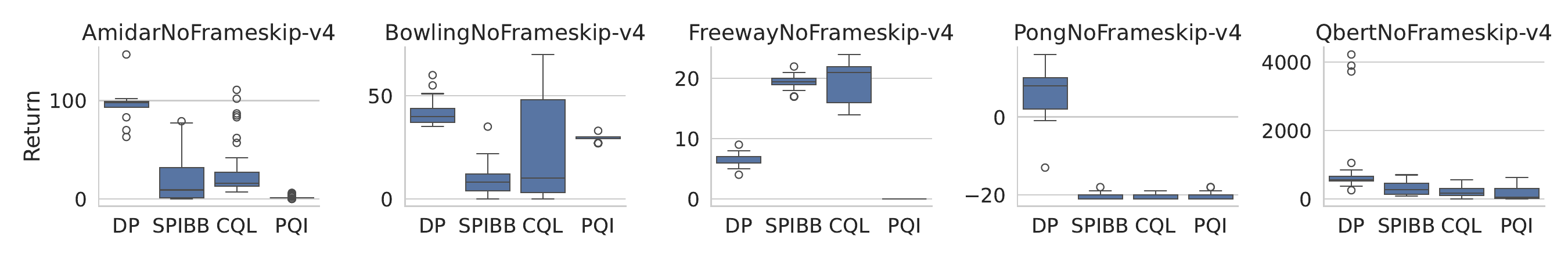}
    \caption{Performance comparison on the Atari domains.}
    \label{fig:atari_results_appendix}
\end{figure}

\section{MIMIC Dataset Experimental Details}

\begin{table}[h]
\centering
\caption{Original and Handcrafted Features with Weights and Data Types (MIMIC Dataset)}
\label{tab:mimic_features}
\begin{tabular}{l c c c}
\toprule
\textbf{Feature}                                & \textbf{Weight} & \textbf{Data Type} & \textbf{Present in the Raw Dataset} \\ \midrule
creatinine                                      & 3               & float              & Yes                                 \\
fraction inspired oxygen                        & 15              & float              & Yes                                 \\
lactate                                         & 15              & float              & Yes                                 \\
urine output                                    & 15              & float              & Yes                                 \\
urine output since last action                  & 5               & bool               & No                                  \\
alanine aminotransferase                        & 5               & float              & Yes                                 \\
aspartate aminotransferase                      & 5               & float              & Yes                                 \\
diastolic blood pressure                        & 5               & float              & Yes                                 \\
mean blood pressure                             & 15              & float              & Yes                                 \\
partial pressure of oxygen                      & 3               & float              & Yes                                 \\
systolic blood pressure                         & 5               & float              & Yes                                 \\
gcs                                             & 15              & float              & Yes                                 \\
gcs since last action                           & 5               & bool               & No                                  \\
creatinine ever recorded                        & 3               & bool               & No                                  \\
fraction inspired oxygen ever recorded          & 15              & bool               & No                                  \\
lactate ever recorded                           & 10              & bool               & No                                  \\
alanine aminotransferase ever recorded          & 5               & bool               & No                                  \\
aspartate aminotransferase ever recorded        & 5               & bool               & No                                  \\
partial pressure of oxygen ever recorded        & 3               & bool               & No                                  \\
flag vaso 1 was last action                     & 15              & bool               & No                                  \\
flag vaso 2 was last action                     & 15              & bool               & No                                  \\
flag vaso 3 was last action                     & 15              & bool               & No                                  \\
flag bolus 1 was last action                    & 15              & bool               & No                                  \\
flag bolus 2 was last action                    & 15              & bool               & No                                  \\
flag bolus 3 was last action                    & 15              & bool               & No                                  \\
total vasopressor dose                          & 15              & float              & No                                  \\
total bolus dose                                & 15              & float              & No                                  \\
total vasopressor dose last 8 hours             & 15              & float              & No                                  \\
total bolus dose last 8 hours                   & 15              & float              & No                                  \\ \bottomrule
\end{tabular}
\end{table}

\paragraph{Cohort and Environment Definition}
For the continuous states dataset, the EHR data contains deidentified clinical data of patients admitted to the Beth Israel Deaconess Medical Center ICU unit \citep{johnson2023mimic}. We select a cohort of patients with at least 7 mean arterial pressure (MAP) measurements of less than 65 mmHg (i.e. hypotension) within the first 72 hours of their ICU stay. The states have 11 features, and we add another 18 features as per prior work in the literature \citep{gottesman_interpretable_2020}. The actions correspond to 4 levels intravenous (IV) fluid bolus therapy and 4 levels of vasopressor therapy (total 16 discrete actions). Fluid bolus and vasopressors are the first-line treatments for patients with hypotension in the ICU.
We used 500 trajectories for training (28944 transitions), and 500 trajectories for validation (28710 transitions), and 1000 trajectories for final evaluation the performance of the algorithms (58033 transitions).

\paragraph{Preprocessing}
For our cohort, we selected the following features from the raw data for the analysis: \textit{creatinine}, \textit{fraction inspired oxygen}, \textit{lactate}, \textit{urine output}, \textit{alanine aminotransferase}, \textit{aspartate aminotransferase}, \textit{diastolic blood pressure}, \textit{mean blood pressure}, \textit{partial pressure of oxygen}, \textit{systolic blood pressure}, and \textit{gcs}. These features have been selected for the hypotension management task in prior work \cite{gottesman_interpretable_2020} 

To create the distance function, we used a weighted Euclidean metric with weights derived from \citet{gottesman_interpretable_2020} on this dataset and cohort. The final set of features includes the original features as well as a collection of handcrafted features, which were determined by domain experts to be relevant for the hypotension task. We provide the original set of features and the final set of features, along with their corresponding weights, in Table \ref{tab:mimic_features}. 

To ensure the state-action pairs are appropriately matched during the computation of distances, we assigned a weight of 10,000 to the action dimension when computing the distance between state-action pairs.

\paragraph{Training details}
For DPRL-C, we need to specify hyperparameters $r$ and $\Nmin$. We found $r=10$ to give the best results. We varied the $\Nmin$ in $\set{10, 30, 50}$.
For SPIBB and PQI, we need to provide similar sets of state-actions with sufficient count or density, and we use the same sets as the ones obtained using $r$ and $\Nmin$.

\paragraph{Evaluation}
To evaluate the MIMIC policies, we estimated the behavior policy $\pi_b$ (using ExtraTrees) and the $Q^\pi$ values (using Fitted Q-Evaluation with ExtraTrees). We then use Doubly Robust Off-policy evaluation (DR-OPE) \citep{jiangDoublyRobustPolicy} using these estimates. Note that we do expect $\pi_b$ and $Q^\pi$ estimates to be imperfect in parts of the state space due to the complexity and coverage challenges in this real-world dataset. Nevertheless, we used DR-OPE to mitigate some of the bias of the estimates.
\end{document}